\documentclass[nohyperref]{article}

\usepackage{microtype}
\usepackage{graphicx}
\usepackage{subfigure}
\usepackage{booktabs} 

\usepackage{titletoc}
\usepackage{lipsum}
\usepackage[page,header]{appendix}

\usepackage{nccmath}

\usepackage[accepted]{icml2022}
\usepackage{hyperref}
\hypersetup{
	colorlinks   = true,
	urlcolor     = blue,
	linkcolor    = blue,
	citecolor    = blue
}
\usepackage{comment}

\usepackage{amsmath}
\usepackage{amssymb}
\usepackage{mathtools}
\usepackage{amsthm}

\usepackage[capitalize,noabbrev]{cleveref}

\theoremstyle{plain}
\newtheorem{theorem}{Theorem}[section]

\newtheorem{lemma}[theorem]{Lemma}

\theoremstyle{definition}
\newtheorem{definition}[theorem]{Definition}
\newtheorem{assumption}[theorem]{Assumption}
\theoremstyle{remark}

\usepackage[textsize=tiny]{todonotes}

\icmltitlerunning{Adaptive Best-of-Both-Worlds Algorithms for Heavy-Tailed MAB}

\usepackage{nicefrac}

\usepackage{bbm, multirow, hyperref}
\allowdisplaybreaks
\newcommand{\E}{\operatornamewithlimits{\mathbb E}}
\newcommand{\longbo}[1]{{\color{red}  [\text{Longbo:} #1]}}
\newcommand{\jiatai}[1]{{\color{blue}  [\text{Jiatai:} #1]}}
\newcommand{\daiyan}[1]{{\color{violet}  [\text{Yan:} #1]}}

 \renewcommand{\longbo}[1]{}
 \renewcommand{\jiatai}[1]{}
 \renewcommand{\daiyan}[1]{}

\renewcommand{\bar}{\overline}

\usepackage{footnote}

\begin{document}

\twocolumn[
\icmltitle{Adaptive Best-of-Both-Worlds Algorithm for Heavy-Tailed Multi-Armed Bandits}




\icmlsetsymbol{equal}{*}

\begin{icmlauthorlist}
\icmlauthor{Jiatai Huang}{equal,thu}
\icmlauthor{Yan Dai}{equal,thu}
\icmlauthor{Longbo Huang}{thu}
\end{icmlauthorlist}

\icmlaffiliation{thu}{Institute for Interdisciplinary Information Sciences, Tsinghua University, Beijing, China}

\icmlcorrespondingauthor{Longbo Huang}{longbohuang@tsinghua.edu.cn}

\icmlkeywords{Machine Learning, ICML, heavy-tailed losses, bandits, best-of-both-worlds, adaptive algorithms, }

\vskip 0.3in
]



\printAffiliationsAndNotice{\icmlEqualContribution} 

\begin{abstract}%
  In this paper, we generalize the concept of heavy-tailed multi-armed bandits to adversarial environments, and develop robust best-of-both-worlds algorithms for heavy-tailed multi-armed bandits (MAB), where losses have $\alpha$-th ($1<\alpha\le 2$) moments bounded by $\sigma^\alpha$, while the variances may not exist. Specifically,  
  we design an algorithm \texttt{HTINF}, when the heavy-tail parameters $\alpha$ and $\sigma$ are known to the agent, \texttt{HTINF} simultaneously achieves the optimal regret for both stochastic and adversarial environments, without knowing the actual environment type a-priori. 
  When $\alpha,\sigma$ are unknown,  \texttt{HTINF} achieves a $\log T$-style instance-dependent regret in stochastic cases and $o(T)$ no-regret guarantee in adversarial cases.
  We further develop an algorithm \texttt{AdaTINF}, 
  achieving $\mathcal O(\sigma K^{1-\nicefrac 1\alpha}T^{\nicefrac{1}{\alpha}})$  minimax optimal regret even in adversarial settings,  without prior knowledge on $\alpha$ and $\sigma$. This result matches the known regret lower-bound \cite{bubeck2013bandits}, which assumed a stochastic environment and $\alpha$ and $\sigma$ are both known.
  To our knowledge, the proposed \texttt{HTINF} algorithm is the first to enjoy a  best-of-both-worlds regret guarantee, and \texttt{AdaTINF} is the first algorithm that can adapt to both $\alpha$ and $\sigma$ to achieve optimal gap-indepedent regret bound in classical heavy-tailed stochastic MAB setting and our novel adversarial formulation.
  
\end{abstract}

\section{Introduction}

In this paper, we focus on the multi-armed bandit problem with heavy-tailed losses. Specifically, in our setting, there is an agent facing $K$ feasible actions (called bandit arms) to sequentially make decisions on. 
For each time step $t\in [T]$,\footnote{Throughout the paper, we use $[n]$ to denote the set $\{1,2,\ldots,n\}$.} each arm $i\in [K]$ is associated with a loss distribution $\nu_{t,i}$ which is unknown to the agent. The only constraint on $\nu_{t,i}$ is that the $\alpha$-th moment ($\alpha\in (1,2]$) is bounded by some constant $\sigma^{\alpha}$, i.e., $\E_{\ell \sim \nu_{t,i}}[\lvert\ell\rvert^\alpha]\le \sigma^\alpha$ for all $t\in [T]$ and $i\in [K]$. However, \textit{neither $\alpha$ nor $\sigma$} is known to the agent. 



At each step $t$, the agent picks an arm $i_t$ and observes a loss $\ell_{t,i_t}$ drawn from the distribution $\nu_{t,i_t}$, which is independent to all previous steps. The goal of the agent is to minimize the \textit{pseudo-regret}, which is defined by the expected difference between the loss he suffered and the loss of always pulling the best arm in hindsight (formally defined in \cref{def:pseudo-regret}), where the expectation is taken with respect to the randomness both in the algorithm and the environment.

Prior MAB literature mostly studies settings where the loss distributions are supported on a bounded interval $I$ (e.g., $I=[0,1]$) known to the agent before-hand, which is a special case of our setting where all $\nu_{t,i}$'s are Dirac measures centered within $I$ \cite{seldin2014one,zimmert2019optimal}. By constrast, there is another common existing MAB setting called scale-free MAB \cite{de2014follow,orabona2018scale}, where the range of losses are not known. In this case, the loss range itself can even depend on other scale parameter of the problem instance (e.g., $T$ and $K$) rather than being a constant. Our heavy-tailed setting can be seen as an intermediate setting between bounded-loss MAB and scale-free MAB, where loss feedback can be indefinitely large, but not in a completely arbitrary manner. This setting naturally extends classical MAB settings, including bounded-loss MAB and sub-Gaussian-loss MAB. 

Following the convention of prior MAB literature, we further distinguish the environment into two typical types. Environment of the first type consists with time homogeneous distributions, i.e., for each $i\in [K]$, $\nu_{t,i}=\nu_{1,i}$ holds for all $t\in [T]$. 
We call them \textit{stochastic} environments. 
\citet{bubeck2013bandits} 
proved that, for heavy-tailed stochastic bandits, even when $\alpha$ and $\sigma$ are both known to the agent, an $\Omega(\sigma K^{1-\nicefrac 1\alpha}T^{\nicefrac{1}{\alpha}})$ instance-independent regret and $\Omega(\sigma^{\frac{\alpha}{\alpha-1}}\sum_{i\ne i^\ast}\log T\Delta_i^{-\frac{1}{\alpha-1}})$ instance-dependent regret is unavoidable, where $i^\ast$ denotes the optimal arm in hindsight, 
and $\Delta_i\triangleq \E_{\ell\sim \nu_{1,i}}[\ell]-\E_{\ell\sim \nu_{1,i^\ast}}[\ell]$ is the sub-optimally gap between $i$ and $i^\ast$. They also designed an algorithm that matches these lower-bounds  up to logarithmic factors when both $\alpha$ and $\sigma$ are known.

In the second type of environments,  loss distributions can be time inhomogeneous, and we call them \textit{adversarial} environments. To our knowledge, no previous work studied similar adversarial heavy-tailed MAB problems. It can be seen that the instance-independent lower-bound $\Omega(\sigma K^{1-\nicefrac 1\alpha}T^{\nicefrac{1}{\alpha}})$ for stochastic heavy-tailed MAB proved by \citet{bubeck2013bandits} is also a lower-bound for this adversarial extension.


In this paper, we develop algorithms for heavy-tailed bandits in \textit{both} stochastic and adversarial cases. 
In contrast to existing (stochastic) heavy-tailed MAB algorithms \cite{bubeck2013bandits,lee2020optimal} that heavily use well-designed mean estimators for heavy-tailed distributions, our algorithms are mainly designed based on the Follow-the-Regularized-Leader (FTRL) framework, which has been applied in a number of adversarial MAB works \cite{zimmert2019optimal,seldin2017improved}. 
Our proposed algorithms enjoy optimal or near-optimal regret guarantees and require much less prior knowledge compared to prior works. When $\sigma, \alpha$ are known before-hand, our algorithm \textit{matches} existing gap dependent and independent regret lower-bounds, while previous algorithms suffer extra $\log$-factors (check \cref{tab:table-regret} for a comparison). Finally, we propose an algorithm with $\mathcal O(\sigma K^{1-\nicefrac 1\alpha}T^{\nicefrac{1}{\alpha}})$ regret even when $\sigma, \alpha$ are \textit{both unknown}, which shows the existing $\Omega(\sigma K^{1-\nicefrac 1\alpha}T^{\nicefrac{1}{\alpha}})$ lower-bound is tight even when all prior knowledge on $\sigma, \alpha$ is absent.

\begin{table*}[t]
\begin{minipage}{\textwidth}
    \caption{An overview of the proposed algorithms and  related works.}
    \label{tab:table-regret}
    \vskip 0.15in
    \begin{savenotes}
    \renewcommand{\arraystretch}{1.5}
    \resizebox{\textwidth}{!}{%
    \begin{tabular}{|c|c|c|c|}\hline
    Algorithm & Loss Type & Prior Knowledge & Total Regret \\\hline
    \multirow{2}{*}{\shortstack{Lower-bounds\\\cite{bubeck2013bandits}}} & \multirow{2}{*}{Stochastic\footnote{As discussed in \Cref{sec:setting}, the instance-independent lower bounds automatically apply to adversarial settings, and the main result of this paper shows that it is indeed tight even for adversarial settings.}} & \multirow{2}{*}{$\alpha,\sigma$} & $\Omega\left (\sigma^{\frac{\alpha}{\alpha-1}}\sum_{i\ne i^\ast}\Delta_i^{-\frac{1}{\alpha-1}}\log T\right )$ \\\cline{4-4}
    & & & $\Omega\left (\sigma K^{1-\nicefrac 1\alpha}T^{\nicefrac 1\alpha}\right )$ \\\hline
    \multirow{3}{*}{\shortstack{\texttt{RobustUCB}\\\cite{bubeck2013bandits}}} & \multirow{3}{*}{Stochastic} & \multirow{3}{*}{$\alpha,\sigma$} & $\mathcal O\left (\sum_{i\ne i^\ast}(\frac{\sigma^\alpha}{\Delta_i})^{\frac{1}{\alpha-1}}\log T\right )$ \textbf{(optimal)} \\\cline{4-4}
    & & & \shortstack{$\mathcal O\left (\sigma(K\log T)^{1-\nicefrac 1\alpha}T^{\nicefrac 1\alpha}\right )$\\\textbf{(sub-optimal for $\log T$ factors)}} \\\hline
     \multirow{2}{*}{\citet{lee2020optimal}} &  \multirow{2}{*}{Stochastic} & \multirow{2}{*}{$\alpha$; require $\mu_i \in [0,1]$} & $\mathcal O\left (K^{1-\nicefrac 1\alpha}T^{\nicefrac 1\alpha}\log K\right )$\footnote{\citet{lee2020optimal} regarded $\sigma$ as a constant when stating their regret bounds. By designing different estimators, they also gave various instance-dependent bounds, each with $(\log T)^{\frac{\alpha}{\alpha-1}}$ (sub-optimal) dependency on $T$. One can check Table 1 in their paper for more details.} \\
     & & & \textbf{(sub-optimal for $\log K$ factors)}\\\hline
    \multirow{4}{*}{\shortstack{\texttt{$\nicefrac 1 2$-Tsallis-INF}\\\cite{zimmert2019optimal}}} & \multirow{2}{*}{SCA-unique\footnote{Abbreviation for stochastically constrained adversarial settings with a unique optimal arm.}} & \multirow{4}{*}{\shortstack{ require  $\alpha=2$ and\\$[0,1]$-bounded losses}} & $\mathcal O\left( \sum\limits_{i\ne i^\ast} \frac{1}{\Delta_i}\log T\right)$ \\
    & & & \textbf{(optimal for $\alpha=2,\sigma=1$ case)}\\\cline{2-2}\cline{4-4}
    & Adversarial &  & ${\mathcal O}\left(\sqrt{KT}\right )$ \textbf{(optimal for $\alpha=2,\sigma=1$ case)}\\\hline
    \multirow{2}{*}{\shortstack{\texttt{HTINF} \textbf{(ours)}}} & SCA-unique & \multirow{2}{*}{$\alpha, \sigma$} & $\mathcal O\left( \sum\limits_{i\ne i^\ast} (\frac{\sigma^\alpha}{\Delta_i})^{\frac 1 {\alpha - 1}}\log T\right)$ \textbf{(optimal)}\\
    \cline{2-2}\cline{4-4}
    & \multirow{1}{*}{Adversarial} &  & ${\mathcal O}\left(\sigma K^{1-\nicefrac 1\alpha} T^{\nicefrac 1 \alpha} \right )$ \textbf{(optimal)}\\\hline
    \multirow{2}{*}{\shortstack{Optimistic \texttt{HTINF} \textbf{(ours)}}} & SCA-unique & \multirow{2}{*}{None} & $\mathcal O\left(\sum_{i\ne i^\ast}\left (\frac{\sigma^{2\alpha}}{\Delta_i^{3-\alpha}}\right )^{\frac{1}{\alpha-1}}\log T\right )$ \\
    \cline{2-2}\cline{4-4} & Adversarial & & $\mathcal O(\sigma^\alpha K^{\frac{\alpha-1}{2}}T^{\frac{3-\alpha}{2}})$\\\hline    \multirow{1}{*}{\shortstack{\texttt{AdaTINF} \textbf{(ours)}}} & Adversarial & None\footnote{Though the time horizon $T$ is assumed to be known in \cref{alg-AdaTINF}, it is in fact non-essential for \texttt{AdaTINF}. The removal of $T$, via a usual doubling trick, will not cause extra factors. Check 
    \cref{sec:remove T in AdaTINF} for more discussions. 
    } & ${\mathcal O}\left( \sigma K^{1-\nicefrac{1}\alpha} T^{\nicefrac 1 \alpha} \right)$  \textbf{(optimal)}\\\hline
    \end{tabular}}
    \end{savenotes}
\end{minipage}
\end{table*}

\subsection{Our Contributions}
We first introduce a novel adversarial MAB setting where losses are heavy-tailed, which generalizes the existing heavy-tailed stochastic MAB setting and scalar-loss adversarial MAB setting. Three novel algorithms are proposed. \texttt{HTINF} enjoys an optimal best-of-both-worlds regret guarantee when $\alpha, \sigma$ are known. Without the knowledge of $\alpha,\sigma$, \texttt{OptTINF} guarantees $o(T)$ adversarial regret (a.k.a. ``\textit{no-regret} guarantee'') and $\mathcal O(\log T)$ gap-dependent bound for stochastically constrained environments.   \texttt{AdaTINF} guarantees \textit{minimax optimal} $\mathcal O(\sigma K^{1-\nicefrac 1\alpha}T^{\nicefrac{1}{\alpha}})$ adversarial regret.

\subsubsection{Known $\alpha,\sigma$ case}
When $\alpha,\sigma$ are both known to the agent, we provide a novel algorithm called \textbf{H}eavy-\textbf{T}ail Tsallis-INF (\texttt{HTINF}, \cref{alg-HTINF}), based on the Follow-the-Regularized-Leader (FTRL) framework. In \texttt{HTINF}, We introduce a novel skipping technique equipped with an \textit{action-dependent} skipping threshold ($r_t$ in \cref{alg-HTINF}) to handle the heavy-tailed losses, which can be of independent interest.

\texttt{HTINF} enjoys the so-called \textit{best-of-both-worlds} property \cite{bubeck2012best} to achieve ${\mathcal O}\left (\sigma K^{1-\nicefrac 1\alpha}T^{\nicefrac{1}{\alpha}}\right )$ regret in adversarial settings and $\mathcal O \left (\sigma^{\frac{\alpha}{\alpha - 1}} \sum_{i\ne i^\ast}\Delta_i^{-\frac{1}{\alpha-1}}\log T \right )$ regret in stochastically constrained adversarial settings (which contains stochastic cases; see \cref{sec:setting} for definition) \textit{simultaneously}, without  knowing the actual environment type a-priori. 
The claimed regret bounds 
both match the corresponding lower-bounds by \citet{bubeck2013bandits}, showing that these bounds are indeed tight even for our adversarial setting.

\subsubsection{Unknown $\alpha,\sigma$ case}
When the agent does not access to $\alpha$ and $\sigma$, running \texttt{HTINF} \textbf{opt}imistically with $\alpha=2$ and $\sigma=1$ (named \texttt{OptTINF}; \cref{alg-AdaHTINF}) also gives non-trivial regret guarantees. Specifically, we showed that it enjoys a near-optimal regret of $\mathcal O\left (\sum_{i\ne i^\ast}(\frac{\sigma^{2\alpha}}{\Delta_i^{3-\alpha}})^{\frac{1}{\alpha-1}}\log T\right )$ in stochastically constrained adversarial environments 
and $\mathcal O\left (\sigma^\alpha K^{\frac{\alpha-1}{2}}T^{\frac{3-\alpha}{2}}\right )$ regret in adversarial cases, which is still $o(T)$. 

We further present another novel algorithm called \textbf{Ada}ptive \text{T}sallis-\text{INF} (\texttt{AdaTINF}, \cref{alg-AdaTINF}) for heavy-tailed bandits. 
Without knowing the heavy-tail parameters $\alpha$ and $\sigma$ before-hand, \texttt{AdaTINF} is capable of guaranteeing an 
${\mathcal O}\left (\sigma K^{1-\nicefrac 1\alpha}T^{\nicefrac{1}{\alpha}}\right )$ regret in the adversarial setting,  
matching the regret lower-bound from \citet{bubeck2013bandits}. 

To the best of our knowledge, all prior algorithms for MAB with heavy-tailed losses need to know $\alpha$ before-hand. The proposed two algorithms, \texttt{OptTINF} and \texttt{AdaTINF}, are the first algorithms to have the adaptivity for \textit{both} unknown heavy-tail parameters $\alpha$ and $\sigma$, while achieving near-optimal regrets in stochastic or adversarial settings.

\subsection{Related Work}
\textbf{Heavy-tailed losses:} The heavy-tailed (stochastic) bandit model was first introduced by \citet{bubeck2013bandits}, where instance-dependent and independent lower-bounds were given. They designed an algorithm nearly matching these lower-bounds (with an extra $\log T$ factor in the gap-independent regret), when $\sigma,\alpha$ are both known to the agent.
\citet{vakili2013deterministic} derived a tighter upper-bound with $\alpha,\sigma$ and $\min \Delta_i$ all presented to the agent.
\citet{kagrecha2019distribution} gave an algorithm adaptive to $\sigma$ in a pure exploration setting.
\citet{lee2020optimal} got rid of the requirement of $\sigma$, 
yielding  near-optimal regret bounds with a prior knowledge on $\alpha$ only. 
Moreover, all above algorithms built on the \texttt{UCB} framework, which does not directly apply to adversarial environments. One can refer to \cref{tab:table-regret} for a comparison.

Other variations with heavy-tailed losses are also studied in the literature, e.g., linear bandits \cite{medina2016no, xue2020nearly}, contextual bandits \cite{shao2018almost} and Lipschitz bandits \cite{lu2019optimal}. However, none of above algorithms removes the dependency on $\alpha$. 


\textbf{Best-of-both-worlds:} This concept of designing a single algorithm to yield near-optimal regret in both stochastic and adversarial environments was first proposed by \citet{bubeck2012best}. \citet{bubeck2012best, auer2016algorithm, besson2018doubling} designed algorithms that initially run a policy for stochastic settings, and may permanently switch to a policy for adversarial settings during execution. 
\citet{seldin2014one,seldin2017improved,wei2018more,zimmert2019optimal} designed algorithm using the Online Mirror Descent (OMD) or Follow the Regularized Leader (FTRL) framework. Our work falls into the second category. 


\textbf{Adaptive algorithms:} There is a rich literature in deriving algorithms adaptive to the loss sequences, for either full information setting \cite{luo2015achieving,orabona2016coin}, stochastic bandits \cite{garivier2011kl,lattimore2015optimally} or adversarial bandits \cite{wei2018more,bubeck2019improved}. There are also many algorithms that is adaptive to the loss range, so-called `scale-free' algorithms \cite{de2014follow,orabona2018scale,hadiji2020adaptation}. However, as mentioned above, to our knowledge, our work is the first to adapt to heavy-tail parameters.

\section{Notations}
\label{sec:notations}
We use $[N]$ to denote the integer set $\{1,2,\cdots,N\}$. Let $f$ be any strictly convex function defined on a convex set $\Omega\subseteq \mathbb R^K$. For $x,y\in \Omega$, if $\nabla f(x)$ exists, we denote the Bregman divergence induced by $f$ as
\begin{equation*}
D_f(y,x)\triangleq f(y)-f(x)-\langle \nabla f(x),y-x\rangle.
\end{equation*}

We use $f^\ast(y)\triangleq \sup_{x\in \mathbb R^K}\{\langle y,x\rangle-f(x)\}$ to denote the Fenchel conjugate of $f$. Denote the $K-1$-dimensional probability simplex by $\triangle_{[K]}=\{x\in \mathbb R^K_+\mid x_1+x_2+\cdots+x_K=1\}$. We use $\mathbf e_i\in \triangle_{[K]}$ to denote the vector whose $i$-th coordinate is $1$ and others are $0$.

Let $\bar f$ denote the restriction of $f$ on $\triangle_{[K]}$, i.e.,
\begin{equation*}
\bar f(x)=\begin{cases}f(x),&x\in \triangle_{[K]}\\\infty,&x\notin \triangle_{[K]}\end{cases}.
\end{equation*}

Let $\mathcal E$ be a random event, we use $\mathbbm 1[\mathcal E]$ to denote the indicator of $\mathcal E$, which equals $1$ if $\mathcal E$ happens, and $0$ otherwise.

\section{Problem Setting}
\label{sec:setting}

We now introduce our formulation of the heavy-tailed MAB problem. Formally speaking, there are $K\ge 2$ available arms indexed from $1$ to $K$, and $T\ge 1$ time slots for the agent to make decisions sequentially. $\{\nu_{t,i}\}_{t\in [T], i\in [K]}$ are $T\times K$ probability distributions over real numbers, which are fixed before the game starts and unknown to the agent (i.e., obliviously adversely chosen). Instead of the usual assumption of bounded variance or even bounded range, we only assume that they are \textit{heavy-tailed}, as follows.
\begin{assumption}[Heavy-tailed Losses Assumption]\label{assump:heavy-tail}
The $\alpha$-th moment of all loss distributions $\{\nu_{t,i}\}$ are bounded by $\sigma^\alpha$ for some constants $1 < \alpha \le 2$ and $\sigma>0$, i.e.,
\begin{equation*}
    \E_{\ell\sim \nu_{t,i}}[\lvert \ell\rvert^\alpha]\le \sigma^\alpha,\quad \forall t\in [T], i\in [K].
\end{equation*}
\end{assumption}

In this paper, we will discuss how to design algorithms for two different cases: $\alpha$ and $\sigma$ are known before-hand or oblivious (i.e., fixed before-hand but unknown to the agent). We denote by $\mu_{t,i}\triangleq \E_{x\sim \nu_{t,i}}[x]$ the individual mean loss for each arm and $\mu_t\triangleq (\mu_{t,1},\mu_{t,2},\ldots,\mu_{t,K})$ the mean loss vector at time $t$, respectively.

At the beginning of each time slot $t$, the agent needs to choose an action $i_t \in [K]$. At the end of time slot $t$, the agent will receive and suffer a loss $\ell_{t,i_t}$, which is guaranteed to be an independent sample from the distribution $\nu_{t,i_t}$. The agent is allowed to make the decision $i_t$ based on all history actions $i_1,\ldots, i_{t-1}$, all history feedback $\ell_{1,i_1},\ldots, \ell_{t-1, i_{t-1}}$, and any amount of private randomness of the agent.

The objective of the agent is to minimize the total loss. Equivalently, the agent aims to minimize the following \textit{pseudo-regret} defined by \citet{bubeck2012best} (also referred to as the regret in this paper for simplicity):
\begin{definition}[Pseudo-regret]\label{def:pseudo-regret}
We define
\begin{align}
\mathcal R_T&\triangleq \max\limits _{i\in [K]}\E\left [\sum_{t=1}^T \ell_{t,i_t}-\sum_{t=1}^T \ell_{t,i}\right ]\nonumber\\ &=\max\limits _{i\in [K]}\E\left [\sum_{t=1}^T \mu_{t,i_t}-\sum_{t=1}^T \mu_{t,i}\right ]\label{eq:pseudo-regret}
\end{align}
to be the \textbf{pseudo-regret} of an MAB algorithm, where the expectation is taken with respect to randomness from both the algorithm and the environment. 
\end{definition}

In the remaining of this paper, we will use $\mathcal F_t \triangleq \sigma(i_1,\cdots,i_t, \ell_{1,i_1},\cdots, \ell_{t,i_t})$ to denote the natural filtration of an MAB algorithm execution.

\subsection{Stochastically Constrained Environments}
\begin{definition}[Stochastic Environments]
If, for each arm $i\in [K]$, all $T$ loss distributions $\nu_{1,i},\nu_{2,i},\ldots,\nu_{T,i}$ are identical, we call such environment a stochastic environment.
\end{definition}

A more general setting is called \textit{stochastically constrained adversarial} setting \cite{wei2018more}, defined as follows.
\begin{definition}[Stochastically Constrained Adversarial Environments]\label{def:SCA env}
If, there exists an \textit{optimal} arm $i^\ast\in [K]$ and mean gaps $\Delta_i \ge 0$ such that for all $t\in [T]$, we have $\mu_{t,i} - \mu_{t,i^\ast} \ge \Delta_i$ for all $i\ne i^\ast$, we call such environment a  stochastically constrained adversarial environment.
\end{definition}

It can be seen that stochastic problem instances are special cases of stochastically constrained adversarial instances. 
Hence, in this paper, we study this more general setting instead of stochastic cases. As in \citet{zimmert2019optimal}, we make the following assumption. 

\begin{assumption}[Unique Optimal Arm Assumption]\label{assump:unique best arm}
In stochastically constrained adversarial environments, $i^\ast$ is the unique best arm throughout the process, i.e.,
\begin{equation*}
    \Delta_i>0,\quad \forall i\ne i^\ast.
\end{equation*}
\end{assumption}

\textbf{Remark.} The existence of a unique optimal arm is a common assumption in MAB and RL literature leveraging FTRL with Tsallis entropy regularizers \citep{zimmert2019optimal,erez2021best,jin2020simultaneously,jin2021best}. Recently, \citet{ito2021parameter} gave a new analysis of Tsallis-INF's logarithmic regret on stochastic MAB instances without this assumption. It is an interesting future work to figure out whether it is doable in our heavy-tailed losses setting. 


\subsection{Adversarial Environments}
In contrast, an environment without any extra requirement is called an adversarial environment. We denote the best arm(s) in hindsight by $i^\ast$, i.e., the $i\in [K]$ that makes the expectation in Eq. (\ref{eq:pseudo-regret}) maximum. We make the following assumption on the losses of arm $i^\ast$.

\begin{assumption}[Truncated Non-nagative Losses Assumption]\label{assump:truncated non-negative}
There exists an optimal arm $i^\ast$ such that $\ell_{t,i^\ast}$ is \textit{truncated non-negative} for all $t\in [T]$.
\end{assumption}

In the assumption, the truncated non-negative property is defined as follows. 
\begin{definition}[Truncated Non-negativity]
\label{def:truncated-nonnegative}
A random variable $X$ is truncated non-negative, if for any $M\ge 0$,
\begin{equation*}
    \E\left[ X \cdot \mathbbm 1[\lvert X \rvert > M ]\right] \ge 0.
\end{equation*}
\end{definition}
\textbf{Remark.} This truncated non-negative requirement is \textit{strictly weaker} than the common non-negative losses assumption in MAB literature, especially works fitting in the FTRL framework \cite{auer2002nonstochastic,zimmert2019optimal}. Intuitively, truncated non-negativity forbids the random variable to hold too much mass on its negative part, but it can still have negative outcomes. 

\section{Static Algorithm: \texttt{HTINF}}\label{sec:HTINF}
In this section, we first present an algorithm achieving optimal regrets when $\alpha,\sigma$ are both known before-hand, and then extend it to the unknown $\alpha,\sigma$ case.
\subsection{Known $\alpha,\sigma$ Case}
For the case where both $\alpha$ and $\sigma$ are known a-priori, we present a FTRL algorithm with the $\frac 1\alpha$-Tsallis entropy function $\Psi(\mathbf x)=-\alpha\sum_{i=1}^K x_i^{\nicefrac 1\alpha}$ \cite{tsallis1988possible,abernethy2015fighting,zimmert2019optimal} as the regularizer. We pick $\eta=t^{-\nicefrac 1\alpha}$ as the  learning rate of the FTRL algorithm. Importance sampling is used to construct estimates $\hat \ell_t$ of the true loss feedback vector $\ell_t$.

In this algorithm, to handle the heavy-tailed losses, we designed a novel \textit{skipping} technique with action-dependent threshold $r_t \propto \eta_t^{-1}x_{t,i_t}^{\nicefrac 1\alpha}$ at time slot $t$, i.e., the agent simply discards those time slots with the absolute value of the loss feedback more than $r_t$. Note that this skipping criterion with dependency on $i_t$ is properly defined, 
for it is checked
\textit{after} deciding the arm $i_t$ and receiving the feedback. To decide $x_t$, the probability to pull each arm in a new time step, we pick the best mixed action $x$ against the sum of all  non-skipped estimated loss $\hat \ell_t$'s, in a regularized manner. 
The pseudo-code of the algorithm is presented in \cref{alg-HTINF}.

\renewcommand{\algorithmicrequire}{\textbf{Input:}}
\renewcommand{\algorithmicensure}{\textbf{Output:}}
\begin{algorithm}[htb]
\caption{\texttt{Heavy-Tail Tsallis-INF} (\texttt{HTINF})}
\label{alg-HTINF}
\begin{algorithmic}[1]
\REQUIRE{Number of arms $K$, heavy-tail parameters $\alpha$ and $\sigma$}
\ENSURE{Sequence of actions $i_1,i_2,\cdots,i_T\in [K]$}
\FOR{$t=1,2,\cdots$}
    \STATE{Calculate policy with learning rate $\eta_t^{-1}=\sigma t^{\nicefrac 1\alpha}$; Pick the regularizer $\Psi(x)=-\alpha \sum_{i=1}^K x_i^{\nicefrac 1\alpha}$:}
     \begin{equation*}
        x_t\gets\operatornamewithlimits{argmin}_{x\in \triangle_{[K]}}\left (\eta_t\sum_{s=1}^{t-1} \langle \hat \ell_s,x \rangle +\Psi(x) \right )
    \end{equation*}
    \STATE Sample new action $i_t\sim x_t$.
    \STATE Calculate the skipping threshold $r_t\gets \Theta_\alpha \eta_t^{-1} x_{t,i_t}^{\nicefrac 1\alpha}$ where $\Theta_\alpha = \min\{1 - 2^{-\frac{\alpha - 1}{2\alpha - 1}}, (2 - \frac 2 \alpha)^{\frac 1 {2 - \alpha}}\}$. \label{line:skipping threshold in HTINF}
    \STATE Play according to $i_t$ and observe loss feedback $\ell_{t,i_t}$.
    \IF{$\lvert\ell_{t,i_t}\rvert>r_t$}
        \STATE$\hat \ell_t \gets \mathbf 0$.
    \ELSE
        \STATE{Construct weighted importance sampling loss estimator $\hat \ell_{t,i} \gets \frac{\ell_{t,i}}{x_{t,i}}\mathbbm 1[i=i_t]$, $\forall i\in [K]$.}
    \ENDIF
\ENDFOR
\end{algorithmic}
\end{algorithm}

The performance of \cref{alg-HTINF} is presented in the following Theorem~\ref{thm:HTINF main theorem}. The proof is sketched in \cref{sec:proof sketch of HTINF}. For a detailed formal proof, see \cref{sec:formal proof of HTINF}.

\begin{theorem}[Performance of \texttt{HTINF}]\label{thm:HTINF main theorem}
If Assumptions \ref{assump:heavy-tail} and \ref{assump:truncated non-negative} hold, we have the following \textbf{best-of-both-worlds} style regret guarantees.
\begin{enumerate}
    \item When the environment is adversarial, 
    \cref{alg-HTINF} ensures regret bound
    \begin{equation*}
        \mathcal R_T\le \mathcal O\left (\sigma K^{1-\nicefrac{1}{\alpha}}T^{\nicefrac 1\alpha}\right ).
    \end{equation*}
    
    \item If the environment is stochastically constrained adversarial with a unique optimal arm $i^\ast$, i.e., \cref{assump:unique best arm} holds, then \cref{alg-HTINF} ensures
    \begin{equation*}
        \mathcal R_T\le \mathcal O\left (\sigma^{\frac{\alpha}{\alpha-1}}\sum_{i\ne i^\ast}\Delta_i^{-\frac{1}{\alpha-1}}\log T\right )\footnote{In this big-O notation, we hide an $\exp(\mathcal O(\frac 1 {\alpha - 1}))$ factor. Such factors also appear in prior upper-bounds and lower-bounds on heavy-tailed MAB, see e.g. \cite{bubeck2013bandits} Theorem 1 and Theorem 3.}.
    \end{equation*}
    
\end{enumerate}
\end{theorem}



The $\mathcal O(\log T)$ instance-dependent bound in \cref{thm:HTINF main theorem} is due to a property similar to the self-bounding property of $\nicefrac 1 2$-Tsallis entropy \cite{zimmert2019optimal}. For $\alpha < 2$, such properties of $\nicefrac 1 \alpha$-Tsallis entropy do not automatically hold, they are made possible by our novel skipping mechanism with action-dependent threshold. 

\subsection{Extending to Unknown $\alpha,\sigma$ Case: \texttt{OptTINF}}
The two hyper-parameters $\sigma, \alpha$ in \cref{alg-AdaTINF} are just set to the true heavy-tail parameters of the loss distributions when they are known before-hand. When the \textit{distributions' heavy-tail parameters} $\alpha,\sigma$ are both unknown to the agent, we can prove that by directly running \texttt{HTINF} with algorithm \textit{hyper-parameters} $\alpha=2$ and $\sigma=1$ (not necessarily equal to the true $\alpha,\sigma$ values) ``optimistically'' as in \cref{alg-AdaHTINF}, one can still  achieve 
$\mathcal O(\log T)$
regret in stochastic case and sub-linear regret in adversarial case.

\begin{algorithm}[htb]
\caption{Optimistic \texttt{HTINF} (\texttt{OptTINF})}
\label{alg-AdaHTINF}
\begin{algorithmic}[1]
\REQUIRE{Number of arms $K$}
\ENSURE{Sequence of actions $i_1,i_2,\cdots,i_T\in [K]$}
\STATE Run \texttt{HTINF} (\cref{alg-HTINF}) with hyper-parameters $\alpha=2$ and $\sigma$=1.
\end{algorithmic}
\end{algorithm}

The performance of \cref{alg-AdaHTINF} is described below. As the analysis is quite similar to that of \cref{alg-HTINF}, we postpone the formal proof to \cref{sec:formal proof of AdaHTINF}. 

\begin{theorem}[Performance of \texttt{OptTINF}]\label{thm:AdaHTINF main theorem}
If Assumptions \ref{assump:heavy-tail} and \ref{assump:truncated non-negative} hold, the following two statements are valid.
\begin{enumerate}
    \item In adversarial cases, \cref{alg-AdaHTINF} achieves
    \begin{equation*}
        \mathcal R_T\le \mathcal O\left (\sigma^\alpha K^{\frac{\alpha-1}{2}}T^{\frac{3-\alpha}{2}} + \sqrt{KT}\right ).
    \end{equation*}
    \item In stochastically constrained adversarial environments with a unique optimal arm $i^\ast$ (\cref{assump:unique best arm}), it ensures
    \begin{equation*}
        \mathcal R_T\le \mathcal O\left (\sigma^{\frac{2\alpha}{\alpha-1}}\sum_{i\ne i^\ast}\Delta_i^{-\frac{3-\alpha}{\alpha-1}}\log T\right ).
    \end{equation*}
    
\end{enumerate}
For both cases, $\sigma$ and $\alpha$ in the regret bounds refer to the true heavy-tail parameters of the loss distributions.
\end{theorem}
\cref{thm:AdaHTINF main theorem} claims that when facing an instance with unknown $1 < \alpha < 2$, \cref{alg-AdaHTINF} still guarantees $O(T^{\frac{3-\alpha} 2})$ ``no-regret'' performance and $\mathcal O(\log T)$ instance-dependent regret upper-bound for stochastic instances. 

\section{Regret Analysis of \texttt{HTINF}}\label{sec:proof sketch of HTINF}
In this section, we sketch the analysis of \cref{alg-HTINF}. By definition, we need to bound
\begin{equation}
    \mathcal R_T(y)\triangleq \sum_{t=1}^T \E\left [\langle x_t-y,\mu_t\rangle\right ]\quad (y\in \triangle_{[K]})
\end{equation}
for the \textit{one-hot} vector $y\triangleq \mathbf e_{i^\ast}$. 
For any $t\in[T], i\in[K]$, let $\mu'_{t,i} \triangleq \E[\ell_{t,i} \mathbbm 1[\lvert \ell_{t,i}\rvert \le r_t]\mid \mathcal F_{t-1}, i_t = i] $. For a given $y$, we decompose $\mathcal R_T(y)$ into two parts:
\begin{align}
    \mathcal R_T(y) & = 
    \E\left [\sum_{t=1}^T \langle x_t-y,\mu_t - \mu'_t\rangle\right ]+\E\left [ \sum_{t=1}^T\langle x_t-y,\mu'_t\rangle\right ] \nonumber \\
    & = \underbrace{\E\left [\sum_{t=1}^T \langle x_t-y,\mu_t - \mu'_t\rangle\right ]}_{\text{skipping gap}}+\underbrace{\E\left [ \sum_{t=1}^T \langle x_t-y,\hat \ell_t\rangle\right ]}_{\text{FTRL error}} \label{eq:regret-partition}
\end{align}

where the last step is due to $\E[\hat \ell_t\mid \mathcal F_{t-1}] = \mu_t'$. 
We call the first part the \textit{skipping gap}, and the second, the \textit{FTRL error}. 

In the following sections, we will show that both parts can be controlled and transformed into expressions similar to the bounds with self-bounding properties in \cite{zimmert2019optimal}, guaranteeing best-of-both-worlds style regret upper-bounds. Therefore, the design of \texttt{HTINF} and our new analysis generalizes the self-bounding property of \cite{zimmert2019optimal} from $1/2$-Tsallis entropy regularizer to general $\alpha$-Tsallis entropy regularizers where $1/2 \le \alpha < 1$.

\subsection{To Control the Skipping Gap}
\label{sec:HTINF-skipped}

To control the skipping gap part, notice that for all $t\in [T], i\in [K]$, we can bound
\begin{align*}
    \mu_{t,i}-\mu_{t,i}'&=\E\left [\lvert \ell_{t,i}\rvert\mathbbm 1[\lvert \ell_{t,i}\rvert>r_t]\mid \mathcal F_{t-1},i_t=i\right ]\\
    &\le \E\left [\lvert \ell_{t,i}\rvert^\alpha r_t^{1-\alpha}\mid \mathcal F_{t-1},i_t=i\right ]\\
    &\le \sigma^\alpha r_t^{1-\alpha}=\Theta_\alpha^{1-\alpha}\sigma t^{\nicefrac 1\alpha-1}x_{t,i}^{\nicefrac 1\alpha-1}
\end{align*}
where $\Theta_\alpha$ is a factor in $r_t$ and only dependent on $\alpha$, as defined in Line 4 of \cref{alg-HTINF}. Moreover, by \cref{assump:truncated non-negative}, $\mu_{t,i^\ast}-\mu'_{t,i^\ast}\ge 0$ a.s. Summing over $i$ and $t$ gives
\begin{align}
    \sum_{t=1}^T \langle x_t - \mathbf e_{i^\ast}, \mu_t - \mu'_t \rangle &\le \Theta_\alpha^{1-\alpha}\sigma\sum_{t=1}^T\sum_{i\ne i^\ast}t^{\nicefrac 1 \alpha - 1} x_{t,i}^{\nicefrac 1 \alpha} \nonumber \\
    & \le 5\sigma\sum_{t=1}^T\sum_{i\ne i^\ast}t^{\nicefrac 1 \alpha - 1} x_{t,i}^{\nicefrac 1 \alpha} \label{eq:skip-times-bound-with-x} \\
    & \le 10\sigma (T+1)^{\nicefrac 1 \alpha} K^{1-\nicefrac 1 \alpha}. \label{eq:skip-times-bound} 
\end{align}
\longbo{explain the steps}

\subsection{To Control the FTRL Error}
For the FTRL error part, we follow the regular analysis for FTRL algorithms. 
Note that our skipping mechanism is equivalent to plugging in $\hat \ell_t = \mathbf{0}$ for all skipped time step $t$ in a FTRL framework for MAB algorithms. Therefore, due to the definition that $\E[\hat \ell_t]=\mu'_t$, we can leverage most standard techniques on regret analysis of a FTRL algorithm and obtain  following lemma. 

\begin{lemma}[FTRL Regret Decomposition]\label{lem:non-skipped part regret decomposition}
    \begin{align*}
     \sum_{t=1}^T \langle x_t-y,\hat\ell_t\rangle & \le \underbrace{\sum_{t=1}^T (\eta_{t}^{-1} - \eta_{t-1}^{-1})\left( \Psi(y) - \Psi(x_t)\right)}_{\text{Part (A)}} \\
    & \quad +  \underbrace{\sum_{t=1}^T\eta_t^{-1}D_\Psi(x_t, z_t)}_{\text{Part (B)}}
\end{align*}

    where 
\begin{equation*}z_t \triangleq \nabla \Psi^*\left (\nabla \Psi(x_t) - \eta_t \mathbbm 1[\lvert \ell_{t,i_t} \rvert \le r_t](\hat\ell_t - \ell_{t,i_t}\mathbf 1)\right ).
\end{equation*}
\end{lemma}
In \cref{lem:non-skipped part regret decomposition}, $z_t$ is an intermediate action probability-like measure vector (which does not necessarily sum up to $1$) during the FTRL algorithm. Here we leverage a trick of \textit{drifting} the loss vectors \cite{wei2018more} $\hat \ell_t'\triangleq \hat \ell_t-\ell_{t,i_t}\mathbf 1$. Intuitively, one can see that feeding $\hat \ell_t'$ into a FTRL framework will produce exactly the same action sequence as $\hat \ell_t$. 

We then divide this upper-bound in Lemma \ref{lem:non-skipped part regret decomposition} into two parts,  parts (A) and (B), and analyze them separately.

\subsubsection{Bound for Part (A)}
As $y$ is an one-hot vector, we have $\Psi(y)=-\alpha$ for $\Psi(x)=-\alpha \sum_{i=1}^K x_i^{\nicefrac 1\alpha}$. Hence, each summand in part (A) becomes
\begin{align*}
    &\quad \left (\eta_t^{-1}-\eta_{t-1}^{-1}\right )\left (-\alpha+\alpha \sum_{i=1}^K x_i^{\nicefrac 1\alpha}\right )\\
    &\qquad\le 2\sigma \frac 1\alpha t^{\nicefrac 1\alpha-1}\cdot \alpha \sum_{i\ne i^\ast}x_{t,i}^{\nicefrac 1\alpha}
\end{align*}

due to the concavity of $t^{\nicefrac 1\alpha}$ (\cref{lem:t^q-(t-1)^q}) and the fact that $x_{t,i}\le 1$. This further implies
\begin{align}
    \text{(A)} & \le \sum_{t=1}^T 2\sigma t^{\nicefrac 1\alpha-1}\sum_{i\ne i^\ast}x_{t,i}^{\nicefrac 1\alpha} \label{eq:investment-self-bounding} \\
    & \le 4\sigma (T+1)^{\nicefrac 1\alpha} K^{1 - \nicefrac 1\alpha}.\label{eq:investment-self-bounding-gap-independent}
\end{align}
\longbo{explain the steps} 


\subsubsection{Bound for Part (B)}
We can bound the expectation of each summand in part (B) as the following lemma states. \daiyan{Give some explanation}

\begin{lemma}\label{lem:part-B-expectation-on-tilde}
\cref{alg-HTINF} ensures 
\begin{align}
\E[\eta_t^{-1}D_\Psi(x_t, z_t)\mid \mathcal F_{t-1}] & \le 8\sigma t^{\nicefrac{1}{\alpha}-1}\sum_{i\ne i^\ast} x_{t,i}^{\nicefrac{1}{\alpha}} \label{eq:B-expectation-bound}\\
& \le 8\sigma t^{\nicefrac{1}{\alpha}-1}K^{1 - \nicefrac 1 \alpha}\label{eq:B-expectation-bound-no-x}.
\end{align}
\end{lemma}
\longbo{explain briefly?}

\subsection{Combining All Parts}

In order to derive the claimed regret upper-bounds in \cref{thm:HTINF main theorem}, it suffices to plug in the bounds for the terms in Eq. (\ref{eq:regret-partition}) and \cref{lem:non-skipped part regret decomposition}. 

\paragraph{Adversarial Case (Statement 1 in \cref{thm:HTINF main theorem}):} 
To obtain an instance-independent bound for the expected total pseudo-regret $\mathcal R_T$, we can plug inequalities (\ref{eq:skip-times-bound}),  (\ref{eq:investment-self-bounding-gap-independent}) and (\ref{eq:B-expectation-bound-no-x}) into Eq. (\ref{eq:regret-partition}) to obtain 
\begin{align*}
    &\quad \mathcal R_T \le 30 \sigma K^{1-\nicefrac 1 \alpha} (T+1)^{\nicefrac 1 \alpha}.
\end{align*}

\paragraph{Stochastically Constrained Adversarial Case (Statement 2 in \cref{thm:HTINF main theorem}):}

To obtain an instance-dependent bound for $\mathcal R_T$, we leverage the arm-pulling probability $\{x_t\}$ dependent bounds (\ref{eq:investment-self-bounding}) and (\ref{eq:B-expectation-bound}) for the FTRL part of $\mathcal R_T$. After plugging them together with (\ref{eq:skip-times-bound-with-x}) into (\ref{eq:regret-partition}), we see that 
\begin{align}
    & \quad \mathcal R_T \le \E\bigg[ \sum_{t=1}^T\sum_{i\ne i^\ast} \underbrace {15 \sigma\left(\frac 1 t\right)^{1 - \nicefrac 1\alpha}x_{t,i}^{\nicefrac 1\alpha}}_{\triangleq s_{t,i}}\bigg]. \label{eq:stochastic-pre-AMGM}
\end{align}

We further apply the inequality of arithmetic and geometric means (AM-GM inequality) to $s_{t,i}$, as
\begin{align*}
s_{t,i} &=\left (\frac {\alpha\Delta_i}2 x_{t,i}\right )^{\frac 1\alpha}\left[ \left(\frac {\alpha\Delta_i} 2\right)^{-\frac 1 {\alpha - 1}} \left(\frac{30\sigma}{\alpha}\right)^{\frac \alpha {\alpha - 1}} \frac 1 t \right]^{\frac {\alpha - 1} \alpha}\\
    &\le \frac{\Delta_i}{2}x_{t,i}+ \frac{\alpha-1}{\alpha}\left(\frac \alpha 2\right)^{-\frac 1 {\alpha - 1}} \left(\frac{30\sigma}{\alpha}\right)^{\frac \alpha {\alpha - 1}} \Delta_i^{-\frac 1 {\alpha - 1}} \frac 1 t.
\end{align*}

By noticing the fact that $\sum_{t\in[T]}\sum_{i\ne i^\ast} \Delta_i\E [x_{t,i}] \le \mathcal R_T$ (\cref{lem:sum x Delta to regret}), Eq. (\ref{eq:stochastic-pre-AMGM}) solves to
\begin{align*}
\mathcal R_T & \le \frac{2\alpha-2}{\alpha}\left(\frac \alpha 2\right)^{-\frac 1 {\alpha - 1}} \left(\frac{30\sigma}{\alpha}\right)^{\frac \alpha {\alpha - 1}} \\ &\quad \cdot \sum_{i\ne i^\ast} \Delta_i^{-\frac 1 {\alpha - 1}} \ln\left(T+1\right) \\
    & = \exp\left(\mathcal O\left(\frac 1 {\alpha - 1}\right)\right)\sigma^{\frac \alpha {\alpha - 1}}\sum_{i\ne i^\ast} \Delta_i^{-\frac 1 {\alpha - 1}} \ln\left(T+1\right).
\end{align*}

\longbo{explain briefly how you solve it}

\section{Adaptive Algorithm: \texttt{AdaTINF}}\label{sec:AdaTINF}


In this section, our main goal is to achieve minimax optimal regret bounds for adversarial settings, without any knowledge about $\alpha,\sigma$. Instead of estimating $\alpha$ and $\sigma$ explicitly, which can be challenging, our key idea is to leverage a trade-off relationship between Part (A) and Part (B) in the FTRL error part (defined in \cref{lem:non-skipped part regret decomposition}),  to balance the two parts dynamically.

To achieve a balance, we use a \textit{doubling trick} to tune the learning rates and skipping thresholds, which has been adopted in the literature to design adaptive algorithms (see, e.g., \citet{wei2018more}). The formal procedure of \texttt{AdaTINF} is given in \cref{alg-AdaTINF}, with the crucial differences between \cref{alg-HTINF} highlighted in blue texts.

It can be seen as \texttt{HTINF} equipped with a multiplier to  both learning rates and skipping thresholds, maintained at running time, as
\begin{equation*}
    \eta_t^{-1}=\lambda_t \sqrt t,\quad r_t=\lambda_t \Theta_2 \sqrt{t}\sqrt{x_{t,i_t}},\quad \forall 1\le t\le T,
\end{equation*}
where $\lambda_t$ is the doubling magnitude for the $t$-th time slot.

\begin{algorithm}[htb]
\caption{\texttt{Adaptive Tsallis-INF} (\texttt{AdaTINF})}
\label{alg-AdaTINF}
\begin{algorithmic}[1]
\REQUIRE{Number of arms $K$, time horizon $T$}
\ENSURE{Sequence of actions $i_1,i_2,\cdots,i_T\in [K]$}
\STATE Initialize $\color{blue} J\gets 0$, $\color{blue} S_0\gets 0$
\FOR{$t=1,2,\cdots$}
    \STATE $\color{blue} \lambda_t\gets 2^J$
    \STATE{Calculate policy with learning rate $\eta_t^{-1}={\color{blue} \lambda_t}\sqrt t$ and regularizar $\Psi(x)=-2 \sum_{i=1}^K x_i^{\nicefrac 12}$:}
     \begin{equation*}
        x_t\gets\operatornamewithlimits{argmin}_{x\in \triangle_{[K]}}\left (\eta_t\sum_{s=1}^{t-1} \langle \hat \ell_s,x \rangle +\Psi(x) \right )
    \end{equation*}
    \STATE Decide action $i_t\sim x_t$, calculate $r_t\gets {\color{blue} \lambda_t} (1 - 2^{-\nicefrac 1 3}) \sqrt t\sqrt{x_{t,i_t}}$.
    \STATE Play according to $i_t$ and observe loss feedback $\ell_{t,i_t}$.
    \IF{$\lvert\ell_{t,i_t}\rvert>r_t$}
        \STATE{$\hat \ell_t\gets \mathbf 0$}
        \STATE $\color{blue} c_t \gets \ell_{t,i_t}$
    \ELSE
        \STATE{Construct weighted importance sampling loss estimator $\hat \ell_{t,i}=\frac{\ell_{t,i}}{x_{t,i}}\mathbbm 1[i=i_t]$, $\forall i\in [K]$.}
        \STATE $\color{blue} c_t \gets 2\eta_t x_{t,i_t}^{-\nicefrac 12}\ell_{t,i_t}^2$
    \ENDIF
    \STATE $\color{blue} S_J \gets S_J + c_t$
            \IF
        {{\color{blue} $2^J \sqrt{K(T+1)}<S_J$}\label{line:alg3-doubling-condition}}
            \STATE $\color{blue} J\gets \max\{J+1,\lceil\log_2 (c_t/\sqrt{K(T+1)})\rceil + 1\}$ 
            \STATE $\color{blue}S_J \gets c_t$
        \ENDIF
\ENDFOR
\end{algorithmic}
\end{algorithm}


We briefly explain our design. Suppose, initially, all $\lambda_t$'s are set to a same number $\lambda > 1$ instead of $1$. Then, part (A) will become approximately $\lambda$ times bigger than that under \texttt{HTINF}, while the expected value of part (B) will be scaled by a factor $\lambda^{1-\alpha}<1$. In other words, increasing $\lambda$ enlarges part (A) but makes part (B) smaller. 
Therefore, if we can estimate parts (A) and (B), we can keep them of roughly the same magnitude, by doubling $\lambda$ whenever (A) becomes smaller than (B).

As Eq. (\ref{eq:skip-times-bound-with-x}) and  (\ref{eq:skip-times-bound}) are similar to Eq. (\ref{eq:B-expectation-bound}) and (\ref{eq:B-expectation-bound-no-x}), the skipping gap can be treated similarly to (B). 
Therefore, we also take it into consideration in the doubling-balance mechanism. 
Due to the future-dependent Eq. (\ref{eq:investment-self-bounding}) is hard to estimate, we use the looser Eq. (\ref{eq:investment-self-bounding-gap-independent}) to represent part (A). This stops \cref{alg-AdaTINF} from enjoying an $\mathcal O(\log T)$-style gap-dependent regret. However, it can still guarantee a minimax optimal regret in general case, as described in \cref{thm:AdaTINF main theorem}.

\begin{theorem}[Performance of \texttt{AdaTINF}]\label{thm:AdaTINF main theorem}
If Assumptions \ref{assump:heavy-tail} and \ref{assump:truncated non-negative} hold,  \cref{alg-AdaTINF} ensures a regret of
\begin{equation*}
    \mathcal R_T\le \mathcal O(\sigma K^{1-\nicefrac{1}{\alpha}}T^{\nicefrac 1\alpha} + \sqrt{KT}),
\end{equation*}
which is \textbf{minimax optimal}.
\end{theorem}

The proof is sketched in \cref{sec:proof sketch of AdaTINF}, while the formal version is deferred to \cref{sec:formal proof of AdaTINF}.

\textbf{Remark.} Though $T$ is assumed to be known in \cref{alg-AdaTINF}, the assumption can be removed via another doubling trick without effect to the order of the total regret. Check 
\cref{sec:remove T in AdaTINF} for more details.


\section{Analysis of \texttt{AdaTINF}}\label{sec:proof sketch of AdaTINF}

\longbo{give some highlevel description about the proof steps?}
Since the crucial learning rate multiplier $\lambda_t$ is maintained by an adaptive doubling trick, in the analysis, we will group time slots with equal $\lambda_t$'s into \textit{epochs}. For $j \ge 0$, $\mathcal T_j \triangleq \{t\in [T]\mid \lambda_t = 2^j\}$ are the indices of time slots belonging to epoch $j$. Further denote the first step in epoch $j$ by $\gamma_j \triangleq \min \{t\in \mathcal T_j\}$ and the last one by $\tau_j \triangleq \max \{t\in \mathcal T_j\}$. Without loss of generality, assume no doubling happened at slot $T$, then the final value of $J$ in \cref{alg-AdaTINF} is just the index of the last non-empty epoch. 

We will first show 
\begin{equation*}
    \mathcal R_T \le \mathcal O\left (\E[2^J]\sqrt{KT}\right ).
\end{equation*}
As defined in the pseudo-code, let
\begin{equation*}
    c_t \triangleq 2\eta_t x_{t,i_t}^{-\nicefrac 12}\ell_{t,i_t}^2\mathbbm 1[\lvert\ell_{t,i_t}\rvert
\le r_t] + \ell_{t,i_t}\mathbbm 1[\lvert\ell_{t,i_t}\rvert > r_t].
\end{equation*}
According to the condition to enter a new epoch (Line 15 in \cref{alg-AdaTINF}),
for all $0\le j <J$, if $\mathcal T_j$ is non-empty, $\tau_j$ will cause $S_j>2^j\sqrt{K(T+1)}$. Hence, we have the following conditions:
\begin{align}
    \mathbbm 1[\gamma_j > 1]c_{\gamma_j - 1} + \sum_{t\in \mathcal T_j \setminus \{\tau_j\}} c_t
     & \le 2^{j} \sqrt{K(T+1)}, \label{eq:epoch-bound} \\
    \sum_{t\in \mathcal T_j} c_t & > 2^{j-1} \sqrt{K(T+1)}. \label{eq:epoch-bound-exceed}
\end{align}
\longbo{explain the steps} 

For $j=J$, as no doubling has happened after that, we have
\begin{align}
    \mathbbm 1[\gamma_J > 1]c_{\gamma_J - 1} + \sum_{t\in \mathcal T_J} c_t & \le 2^J \sqrt{K(T+1)}. \label{eq:final-epoch-bound}
\end{align}

Similar to Eq. (\ref{eq:regret-partition}) used in \cref{sec:proof sketch of HTINF}, we begin with the following decomposition of $\mathcal R_T(y)$ for $y=\mathbf e_{i^\ast}$:
\begin{align}
    \mathcal R_T(y)
    & =  \underbrace{\E\left [\sum_{t=1}^T \langle x_t-y,\mu_t - \mu'_t\rangle\right ]}_{\mathcal R_T^s} + \underbrace{\E\left [ \sum_{t=1}^T \langle x_t-y,\hat \ell_t\rangle\right ]}_{\mathcal R_T^f} \label{eq:alg3-partition}
\end{align}
where $\mu'_{t,i} \triangleq \E[\ell_{t,i} \mathbbm 1[\lvert \ell_{t,i}\rvert \le r_t]\mid \mathcal F_{t-1}, i_t = i]$. We still call $\mathcal R_T^s$ the skipping gap and $\mathcal R_T^f$ the FTRL error.

According to \cref{lem:non-skipped part regret decomposition}, we have
\begin{align}
   \hspace{-0.1in} \mathcal R_T^f & \le \E\bigg [\underbrace{\eta_T \max_{x\in \Delta_{[K]}} \Psi(x)}_{\text{Part (A)}}\bigg ] + \E\bigg [\underbrace{\sum_{t=1}^T \eta_t^{-1}D_\Psi(x_t,z_t)}_{\text{Part (B)}}\bigg ] \nonumber \\
   \hspace{-0.1in}  & \le \E[2^J] \sqrt{K(T+1)} + \E\left [\sum_{t=1}^T \eta_t^{-1}D_\Psi(x_t,z_t)\right ] \label{eq:alg3-ftrl}
\end{align}

Similar to \cref{alg-HTINF}, we can show $D_\Psi(x_t,z_t) \le 2\eta_t x_{t,i_t}^{-\nicefrac 12}\ell_{t,i_t}^2\mathbbm 1[\lvert\ell_{t,i_t}\rvert \le r_t]$ for all $t\in [T]$. Moreover, by Assumption \ref{assump:truncated non-negative},  $\mathcal R_T^s\le \E[\sum_{t=1}^T \ell_{t,i_t}\mathbbm 1[\lvert \ell_{t,i_t}\rvert>r_t]]$. Therefore, with the help of Eq. (\ref{eq:epoch-bound}) and (\ref{eq:final-epoch-bound}), we have
\begin{align}
& \quad \mathcal R_T^s+\E[\text{Part (B)}] \nonumber \\
& \le \E\left[\sum_{t=1}^T c_t\right] \nonumber \\
& \le \E\left[2^{J+1}\sqrt{K(T+1)}\right]. \label{eq:alg3-A+C}
\end{align}

Combining Eq. (\ref{eq:alg3-partition}), (\ref{eq:alg3-ftrl}) and (\ref{eq:alg3-A+C}) gives
\begin{equation*}
    \mathcal R_T \le \E[2^J] \cdot 3\sqrt{K(T+1)},
\end{equation*}

Therefore, it remains to bound $\E[2^J]$. When $J=0$, there is nothing to do. Otherwise, consider the second to last non-empty epoch, $\mathcal T_{J'}$. The condition to enter a new epoch also guarantees that $2^J\sqrt{K(T+1)}\le 2^{J'+1}\sqrt{K(T+1)}+4c_{\tau_{J'}}$. Applying Eq. (\ref{eq:epoch-bound-exceed}) to $J'<J$, we obtain
\begin{equation}
    \hspace*{-0.25cm}\mathbbm 1[J\ge 1](2^{J'})^\alpha \sqrt{K(T+1)} \le (2^{J'})^{\alpha-1}2\sum_{t\in \mathcal T_{J'}} c_t, \label{eq:formula-to-solve-E-2^J'}
\end{equation}

After appropriate relaxing the RHS of Eq. (\ref{eq:formula-to-solve-E-2^J'}) and taking expectation of both sides, it solves to the following upper-bound for $\E[\mathbbm 1[J\ge 1]2^{J'}]$:
\begin{lemma}
\label{lem:2^J'}
\cref{alg-AdaTINF} guarantees that 
\begin{equation*}
    \E[\mathbbm 1[J\ge 1]2^{J'}]\le 28\sigma K^{\nicefrac 12-\nicefrac 1\alpha}(T+1)^{\nicefrac 1\alpha-\nicefrac 12}.
\end{equation*}
\end{lemma}

Moreover, we can obtain a bound for $\E[c_{\tau_{J'}}]$ stated as follows:
\begin{lemma}
\label{lem:c_tau_J'}
\cref{alg-AdaTINF} guarantees that 
\begin{align*}
    & \quad \E[\mathbbm 1[J\ge 1]c_{\tau_{J'}}] \\
    & \le 0.1\E[\mathbbm 1[J\ge 1]2^{J'}\sqrt T]+4\E\left [\max\limits_{t\in [T]}\lvert \ell_{t,i_t}\rvert\right ].
\end{align*}
\end{lemma}

Using the fact that $\E[\max_{t\in [T]}\lvert \ell_{t,i_t}\rvert]\le \sigma T^{\nicefrac 1\alpha}$ (\cref{lem:expectation of maximum of n variables}), we conclude that \cref{alg-AdaTINF} has the regret guarantee of
\begin{align*}
    &\quad \mathcal R_T\le  3\E[2^J\sqrt{K(T+1)}] \\
    & \le 3\sqrt{K(T+1)} + 204\sigma K^{1-\nicefrac 1 \alpha} (T+1)^{\nicefrac 1 \alpha} + 12\sigma T^{\nicefrac 1 \alpha}.
\end{align*}

\section{Conclusion}
We propose \texttt{HTINF}, a novel algorithm achieving the optimal instance-dependent regret bound for the stochastic heavy-tailed MAB problem, and the optimal instance-independent regret bound for a more general adversarial setting, without extra logarithmic factors. We also propose \texttt{AdaTINF}, which can achieve the same optimal instance-independent regret even when prior knowledge on heavy-tailed parameters $\alpha,\sigma$ are absent. Our work shows that the FTRL (or OMD) technique can be a powerful tool for designing heavy-tailed MAB algorithm, leading to novel theoretical results that have not been achieved by UCB algorithms.

It is an interesting future work to figure out whether it is possible to design a best-of-both-worlds algorithm without knowning the actual heavy-tail distribution parameters $\alpha$ and $\sigma$.

\section*{Acknowledgment}

This work is supported by the Technology and Innovation
Major Project of the Ministry of Science and Technology of China under Grant 2020AAA0108400 and 2020AAA0108403.

\newpage
\bibliography{references}
\bibliographystyle{icml2022}

\onecolumn
\newpage
\appendix
\renewcommand{\appendixpagename}{\LARGE Supplementary Materials: Proofs and Discussions}
\appendixpage

\startcontents[section]
\printcontents[section]{l}{1}{\setcounter{tocdepth}{2}}


\section{Formal Analysis of \texttt{HTINF} (\cref{alg-HTINF})}\label{sec:formal proof of HTINF}
\subsection{Main Theorem}
In this section, we present a formal proof of \cref{thm:HTINF main theorem}. For the sake of accuracy, we state the regret guarantees without using any big-Oh notations, as follows (which directly implies \cref{thm:HTINF main theorem}).

\begin{theorem}[Regret Guarantee of \cref{alg-HTINF}]\label{thm:HTINF main theorem in appendix}
If Assumptions \ref{assump:heavy-tail} and \ref{assump:truncated non-negative} hold, i.e., the environment is heavy-tailed with parameters $\alpha$ and $\sigma$, and there is an optimal arm whose all losses are truncated non-negative. Then \cref{alg-HTINF} guarantees:
\begin{enumerate}
    \item The regret is no more than
    \begin{equation*}
        \mathcal R_T\le 30 \sigma K^{1-\nicefrac 1 \alpha} (T+1)^{\nicefrac 1 \alpha},
    \end{equation*}
    
    no matter the environment is stochastic or adversarial.
    \item Furthermore, if the environment is stochastically constrained with a unique best arm $i^\ast$, i.e., \cref{assump:unique best arm} holds, then it, in addition, enjoys a regret bound of
    \begin{equation*}
        \mathcal R_T\le \frac{2\alpha-2}{\alpha}\left(\frac \alpha 2\right)^{-\frac 1 {\alpha - 1}} \left(\frac{30\sigma}{\alpha}\right)^{\frac \alpha {\alpha - 1}} \sum_{i\ne i^\ast} \Delta_i^{-\frac 1 {\alpha - 1}} \ln\left(T+1\right).
    \end{equation*}
\end{enumerate}
\end{theorem}
\begin{proof}
Define $\mu'_{t,i} \triangleq \E[\ell_{t,i} \mathbbm 1[\lvert \ell_{t,i}\rvert \le r_t]\mid \mathcal F_{t-1}, i_t = i]$. 
For the given $y=\mathbf e_{i^\ast}\in \triangle_{[K]}$, consider the regret of the algorithm with respect to policy $y$, defined and decomposed as
\begin{equation*}
    \mathcal R_T(y)\triangleq \sum_{t=1}^T \mathbb E[\langle x_t-y,\mu_t\rangle]=\mathbb E\left [\sum_{t=1}^T\langle x_t-y,\mu_t - \mu'_t\rangle\right ]+\mathbb E\left [\sum_{t=1}^T\langle x_t-y,\mu'_t\rangle\right ]\triangleq \mathcal R_T^s(y)+\mathcal R_T^f(y),
\end{equation*}

which we called the \textit{skipped part} and \textit{FTRL part}. For simplicity, we abbreviate the parameter $(y)$ for $\mathcal R_T^s$ and $\mathcal R_T^f$.

As defined in \cref{alg-HTINF}, $\hat \ell_t$ is set to $0$ when $\lvert\ell_{t,i_t}\rvert > r_t$. Hence, by the property of weighted importance sampling estimator (\cref{lem:importance sampler}; note that it is applied to the truncated loss with mean $\mu_{t,i}'$), $\E[\hat \ell_{t,i}\mid\mathcal F_{t-1}]=\mu'_{t,i}$ 
\begin{equation*}
    \mathcal R_T^f=\E\left [\sum_{t=1}^T\langle x_t-y,\hat \ell_t\rangle\right ].
\end{equation*}

For the first term, $\mathcal R_{\mathcal T}^s$, we can bound it using the following two lemmas, whose proof are propounded to next subsection.
\begin{lemma}\label{lem:x_t-y cdot mu_t bound}
For any $1\le t\le T$ and $i\in[K]$, we have
\begin{equation*}
    \mu_{t,i} - \mu'_{t,i} \le \Theta_\alpha^{1-\alpha}\sigma t^{\nicefrac 1 \alpha - 1} x_{t,i}^{\nicefrac 1 \alpha - 1},
\end{equation*}

where $\Theta_\alpha$ is a constant used in \cref{alg-HTINF} that only depends on $\alpha$.
\end{lemma}
\begin{lemma}\label{lem:HTINF skipped slots}
If $i^\ast$ is an optimal arm whose loss feedback are all truncated non-negative, then for $y = \mathbf e_{i^\ast}$, we have
\begin{equation*}
    \mathcal R_T^s(y) \le \E\left[\sum_{t=1}^T \sum_{i\ne i^\ast} x_{t,i} (\mu_{t,i} - \mu'_{t,i})\right].
\end{equation*}
\end{lemma}

Therefore, for $y=\mathbf e_{i^\ast}$ we have
\begin{align}
    \mathcal R_T^s & \le \E\left [\sum_{t=1}^T\sum_{i\ne i^\ast} \Theta_\alpha^{1-\alpha}\sigma t^{\nicefrac 1 \alpha - 1} x_{t,i}^{\nicefrac 1 \alpha}\right] \nonumber \\
    & \stackrel{\text(a)}{\le} \E\left [5\sum_{t=1}^T\sum_{i\ne i^\ast} \sigma t^{\nicefrac 1 \alpha - 1} x_{t,i}^{\nicefrac 1 \alpha}\right] \label{eq:HTINF skipped in appendix, gap-dependent} \\
    & \stackrel{\text(b)}{\le} 5\alpha \sigma K^{ 1 - \nicefrac 1 \alpha} (T+1)^{\nicefrac 1 \alpha}, \label{eq:HTINF skipped in appendix, gap-independent}
\end{align}
where step (a) is due to $\Theta_\alpha^{1-\alpha} \le \Theta_2^{-1} \le 5$ and (b) applies \cref{lem:x^t to K^{1-t}} and \cref{lem:t^t to T^{t+1}}.

Now consider the second term, $\mathcal R_T^f$. Consider the vector $\hat \ell_t'\triangleq \mathbbm 1[\lvert \ell_{t,i_t}\rvert \le r_t](\hat \ell_t-\ell_{t,i_t}\mathbf 1)$. Note that $\langle \hat \ell_t',x\rangle=\langle \hat \ell_t,x\rangle-\mathbbm 1[\lvert \ell_{t,i_t}\rvert \le r_t]\ell_{t,i_t}$ for any vector $x\in \triangle_{[K]}$, so a FTRL algorithm fed with loss vector $\hat \ell_t'$ with produce exactly the same action sequence as another instance fed with $\hat \ell_t$ (as constant terms will never affect the choice of the argmax operator over the simplex). Therefore, we can apply \cref{lem:FTRL regret decomposition} with loss vectors as $\hat \ell_{t}'$, yielding
\begin{equation}
    \sum_{t=1}^T \langle x_t-y,\hat \ell_t\rangle=\sum_{t=1}^T \langle x_t-y,\hat \ell_t'\rangle\le \underbrace{\sum_{t=1}^T \left (\eta_{t}^{-1}-\eta_{t-1}^{-1}\right )(\Psi(y)-\Psi(x_t))}_{\text{Part (A)}}+\underbrace{\sum_{t=1}^T\eta_t^{-1}D_\Psi(x_t,z_t)}_{\text{Part (B)}}\label{eq:HTINF FTRL type 2}
\end{equation}
where $z_t\triangleq \nabla \Psi^\ast(\nabla \Psi(x_t)-\eta_t \hat \ell_t')=\nabla \Psi^\ast \left (\nabla \Psi(x_t)-\eta_t \mathbbm 1[\lvert \ell_{t,i_t}\rvert \le r_t](\hat \ell_t-\ell_{t,i_t}\mathbf 1)\right )$.

Now consider the first term $\sum_{t=1}^T (\eta_t^{-1}-\eta_{t-1}^{-1})(\Psi(y)-\Psi(x_t))$, which is denoted by (A) for simplicity. We have
\begin{lemma}\label{lem:HTINF part (A) in appendix}
For part (A), \cref{alg-HTINF} ensures the following inequality for any one-hot vector $y\in \triangle_{[K]}$:
\begin{equation}\label{eq:HTINF part (A) stochastic in appendix}
    \E[\text{(A)}]=\E\left[\sum_{t=1}^T \E\left [(\eta_{t}^{-1}-\eta_{t-1}^{-1})(\Psi(y)-\Psi(x_t))\mid \mathcal F_{t-1}\right ]\right]\le \E\left [\sum_{t=1}^T 2 \sigma t^{\nicefrac 1\alpha-1}\sum_{i\ne i^\ast}x_{t,i}^{\nicefrac 1\alpha}\right ],
\end{equation}

which further implies
\begin{equation}\label{eq:HTINF part (A) adversarial in appendix}
    \E[\text{(A)}] \le \sum_{t=1}^T 2 \sigma t^{\nicefrac 1\alpha-1}K^{1-\nicefrac 1\alpha}.
\end{equation}
\end{lemma}

For the second term, denoted by (B), we have
\begin{lemma}[Restatement of \cref{lem:part-B-expectation-on-tilde}]\label{lem:HTINF part (B) in appendix}
For Part (B), \cref{alg-HTINF} ensures
\begin{equation}\label{eq:HTINF part (B) stochastic in appendix} \E[\text{(B)}]=\E\left[\sum_{t=1}^T\E[\eta_t^{-1}D_\Psi(x_t, z_t)\mid \mathcal F_{t-1}]\right] \le\E\left [\sum_{t=1}^T 8\sigma t^{\nicefrac{1}{\alpha}-1}\sum_{i\ne i^\ast} x_{t,i}^{\nicefrac{1}{\alpha}}\right ],
\end{equation}

which further implies

\begin{equation}\label{eq:HTINF part (B) adversarial in appendix}
    \E[\text{(B)}] \le\sum_{t=1}^T 8\sigma t^{\nicefrac{1}{\alpha}-1}K^{1 - \nicefrac 1 \alpha}.
\end{equation}
\end{lemma}

Hence, for general cases, due to \cref{eq:HTINF part (A) adversarial in appendix,eq:HTINF part (B) adversarial in appendix} we have 
\begin{equation*}
    \mathcal R_T^f=\E[\text{(A)}]+\E[\text{(B)}]\le \sum_{t=1}^T 10\sigma K^{1-\nicefrac 1\alpha}\sum_{t=1}^T t^{\nicefrac 1\alpha-1}\le 10\alpha \sigma K^{1-\nicefrac 1\alpha}(T+1)^{\nicefrac 1\alpha},
\end{equation*}

where the last inequality comes from \cref{lem:t^t to T^{t+1}}. Therefore, taking (\ref{eq:HTINF skipped in appendix, gap-dependent}) into consideration, we have:
\begin{equation*}
    \mathcal R_T=\mathcal R_T^s+\mathcal R_T^f\le 15\alpha \sigma K^{1-\nicefrac 1\alpha}(T+1)^{\nicefrac 1\alpha} \le 30 \sigma K^{1-\nicefrac 1\alpha}(T+1)^{\nicefrac 1\alpha}.
\end{equation*}

Now, for stochastically constrained adversarial case with unique best arm $i^\ast$ throughout the process, due to Equations (\ref{eq:HTINF skipped in appendix, gap-dependent}) (\ref{eq:HTINF part (A) stochastic in appendix}) and (\ref{eq:HTINF part (B) stochastic in appendix}), we have
\begin{equation*}
    \mathcal R_T= \mathcal R_T^s + \E[\text{(A)}]+\E[\text{(B)}]\le \E\left [\sum_{t=1}^T\sum_{i\ne i^\ast} \underbrace{15\sigma t^{\nicefrac 1\alpha-1} x_{t,i}^{\nicefrac 1\alpha}}_{\triangleq s_{t,i}}\right ].
\end{equation*}

We can then write
\begin{align*}
    s_{t,i}&=\left (\frac \alpha 2\Delta_ix_{t,i}\right )^{\nicefrac 1\alpha}\left[ \left(\frac \alpha 2\right)^{-\frac 1 {\alpha - 1}} \left(\frac{30\sigma}{\alpha}\right)^{\frac \alpha {\alpha - 1}} \Delta_i^{-\frac 1 {\alpha - 1}} \frac 1 t \right]^{\frac {\alpha - 1} \alpha}\\
    &\le \frac{\Delta_i}{2}x_{t,i}+ \frac{\alpha-1}{\alpha}\left(\frac \alpha 2\right)^{-\frac 1 {\alpha - 1}} \left(\frac{30\sigma}{\alpha}\right)^{\frac \alpha {\alpha - 1}} \Delta_i^{-\frac 1 {\alpha - 1}} \frac 1 t
\end{align*}
where the last step uses the inequality of arithmetic and geometric means $a^{\nicefrac 1\alpha} b^{1- \nicefrac 1\alpha} \le \frac 1 \alpha a + \left(1 - \frac 1 \alpha\right) b$. Therefore
\begin{align}
    R_T & \le \E\left[\sum_{t=1}^T \sum_{i\ne i^\ast} \frac{\Delta_i}2 x_{t,i}\right] + \sum_{t=1}^T \sum_{i\ne i^\ast} \frac{\alpha-1}{\alpha}\left(\frac \alpha 2\right)^{-\frac 1 {\alpha - 1}} \left(\frac{30\sigma}{\alpha}\right)^{\frac \alpha {\alpha - 1}} \Delta_i^{-\frac 1 {\alpha - 1}} \frac 1 t\nonumber  \\
    & \le \frac 1 2 R_T + \sum_{i\ne i^\ast} \frac{\alpha-1}{\alpha}\left(\frac \alpha 2\right)^{-\frac 1 {\alpha - 1}} \left(\frac{30\sigma}{\alpha}\right)^{\frac \alpha {\alpha - 1}} \Delta_i^{-\frac 1 {\alpha - 1}}\ln(T+1) \label{eq:sti-bound}
\end{align}
where the last step uses \cref{lem:sum x Delta to regret}. \cref{eq:sti-bound} then solves to
\begin{equation*}
    \mathcal R_T\le \sum_{i\ne i^\ast} \frac{2\alpha-2}{\alpha}\left(\frac \alpha 2\right)^{-\frac 1 {\alpha - 1}} \left(\frac{30\sigma}{\alpha}\right)^{\frac \alpha {\alpha - 1}} \Delta_i^{-\frac 1 {\alpha - 1}}\ln(T+1),
\end{equation*}
as claimed.
\end{proof}

\begin{proof}[Proof of \cref{thm:HTINF main theorem}]
It is a direct consequence of the theorem above.
\end{proof}

\subsection{Proof when Bounding $\mathcal R_T^s$ (\textit{the skipped part})}

\begin{proof}[Proof of \cref{lem:x_t-y cdot mu_t bound}]
Starting from the definition of $\mu'_{t,i}$ and $\mu_{t,i}$, we can write
\begin{align*}
    \mu_{t,i} - \mu'_{t,i} & = \E\left[\ell_{t,i_t} \mid \mathcal F_{t-1}, i_t = i\right] - \E\left[\ell_{t,i_t} \cdot \mathbbm 1[\lvert \ell_{t,i_t}\rvert \le r_t] \mid \mathcal F_{t-1}, i_t = i\right] \\
    & = \E\left[\ell_{t,i_t} \cdot \mathbbm 1[\lvert \ell_{t,i_t}\rvert > r_t ] \mid \mathcal F_{t-1}, i_t = i\right] \\
    & \le \E\left[\lvert \ell_{t,i_t}\rvert \cdot \mathbbm 1[\lvert \ell_{t,i_t}\rvert > r_t ] \mid \mathcal F_{t-1}, i_t = i\right] \\
    & \le \E\left[\lvert \ell_{t,i_t}\rvert ^\alpha r_t^{1-\alpha} \cdot \mathbbm 1[\lvert \ell_{t,i_t}\rvert > r_t ] \mid \mathcal F_{t-1}, i_t = i\right] \\
    & \le \E\left[\lvert \ell_{t,i_t}\rvert ^\alpha r_t^{1-\alpha} \mid \mathcal F_{t-1}, i_t = i\right] \\
    & \stackrel{\text{(a)}}{=} \E\left[\lvert \ell_{t,i}\rvert ^\alpha \Theta_\alpha^{1 - \alpha} \sigma^{1-\alpha} t^{\frac {1-\alpha} \alpha} x_{t,i}^{\frac {1-\alpha} \alpha} \mid \mathcal F_{t-1}\right] \\
    & \le \sigma \Theta_\alpha^{1 - \alpha} t^{\frac {1-\alpha} \alpha} x_{t,i}^{\frac {1-\alpha} \alpha}
\end{align*}
where in step (a) we plug in $r_t = \Theta_\alpha \eta_t^{-1} x_{t,i_t}^{\nicefrac 1 \alpha}$.
\end{proof}

\begin{proof}[Proof of \cref{lem:HTINF skipped slots}]
Recall that $\mu_{t,i} - \mu'_{t,i} = \E\left[\ell_{t,i} \cdot \mathbbm 1[\lvert \ell_{t,i}\rvert > r_t ] \mid \mathcal F_{t-1}, i_t = i\right]$, hence according to our assumption that $\ell_{t,i^\ast}$ is truncated non-negative (\cref{assump:truncated non-negative}), we have $\mu_{t,i^\ast} - \mu'_{t,i^\ast} \ge 0$ a.s., thus when $y = \mathbf e_{i^\ast}$,
\begin{equation*}
    (x_{t,i^\ast} - y) \cdot (\mu_{t,i} - \mu_{t,i^\ast}) = (x_{t,i^\ast} - 1) \cdot (\mu_{t,i} - \mu_{t,i^\ast}) \le 0.
\end{equation*}
Therefore
\begin{align*}
    \mathcal R_T^s(y) & = \E \left[ \sum_{t=1}^T \langle x_t - y, \mu_t - \mu'_t \rangle  \right] \\
    & \le \E \left[ \sum_{t=1}^T \sum_{i\ne i^\ast} \left(x_{t,i} - y_i \right) \cdot \left(\mu_{t,i} - \mu'_{t,i} \right)  \right] \\
    & = \E \left[ \sum_{t=1}^T \sum_{i\ne i^\ast} x_{t,i} \left(\mu_{t,i} - \mu'_{t,i} \right)  \right],
\end{align*}
as claimed.
\end{proof}

\subsection{Proof when Bounding $\mathcal R_T^f$ (\textit{the FTRL part})}

For our purpose, we need a technical lemma stating that the components of $z_t$ are at most a constant times larger than $x_t$'s components.

\begin{lemma}
For any $t\in [T]$ and $i \in [K]$, \cref{alg-HTINF} guarantees that 
\begin{equation*}
    z_{t,i} \le 2^{\frac {\alpha}{2\alpha - 1}}x_{t,i}
\end{equation*}

where $z_t \triangleq \nabla \Psi^*(\nabla \Psi(x_t) - \eta_t \mathbbm 1[\lvert \ell_{t,i_t}\rvert \le r_t](\hat\ell_t - \ell_{t,i_t}\mathbf 1))$.
\end{lemma}
\begin{proof}
If $\lvert \ell_{t,i_t}\rvert > r_t$, then $x_t = z_t$. Otherwise, we denote $\nabla \Psi(x_t)$ by $x_t^*$, and denote $\nabla \Psi(z_t)$ by $z_t^*$, then we have $-x_{t,i}^* = x^{-\frac{\alpha - 1} \alpha}$ and
\begin{align*}
    z_{t,i}^* & = x_{t,i}^* -\eta_t \hat \ell_{t,i} + \eta_t \ell_{t,i_t} \\
    & = \begin{cases}
        x_{t,i}^* -\eta_t \frac {\ell_{t,i}} {x_{t,i}} + \eta_t \ell_{t,i} & i=i_t \\
        x_{t,i^*} +\eta_t \ell_{t,i_t} & i\ne i_t.
    \end{cases}
\end{align*}
If $i=i_t$, we have
\begin{align*}
    - z_{t,i}^* \ge -x_{t,i}^* - \eta_t \frac{\lvert \ell_{t,i}\rvert} {x_{t,i}}
    = x_{t,i}^{-\frac{\alpha - 1}{\alpha}} - \eta_t \frac{\lvert \ell_{t,i}\rvert} {x_{t,i}}
    \ge x_{t,i}^{-\frac{\alpha - 1}{\alpha}} - \Theta_\alpha x_{t,i}^{\frac{1-\alpha}{\alpha}},
\end{align*}
where the last step is due to $\lvert \ell_{t,i_t}\rvert \le r_t$ and our choice of $r_t$ in \cref{alg-HTINF}. Thus
\begin{align*}
    z_{t,i} = (-z_{t,i}^*) ^ {-\frac{\alpha} {\alpha - 1}} \le x_{t,i}(1 - \Theta_\alpha)^{-\frac{\alpha} {\alpha - 1}} \le 2^{\frac {\alpha}{2\alpha - 1}} x_{t,i}
\end{align*}
where the last step is because $\Theta_\alpha \le 1 - 2^{-\frac{\alpha - 1}{2\alpha - 1}}$.

If $i\ne i_t$, we have $-z_{t,i}^* \ge -x_{t,i}^* - \Theta_\alpha x_{t,i_t}^{\nicefrac 1 \alpha}\ge x_{t,i}^{-\frac{\alpha - 1}{\alpha}} - \Theta_\alpha$, thus
\begin{align*}
    z_{t,i} = (-z_{t,i}^*) ^ {-\frac{\alpha} {\alpha - 1}} \le x_{t,i}(1 - \Theta_\alpha x_{t,i}^{\frac{\alpha - 1}{\alpha}})^{-\frac{\alpha} {\alpha - 1}} \le x_{t,i}(1 - \Theta_\alpha)^{-\frac{\alpha} {\alpha - 1}} \le 2^{\frac {\alpha}{2\alpha - 1}} x_{t,i}.
\end{align*}

Combining two cases together gives our conclusion.
\end{proof}

\begin{proof}[Proof of \cref{lem:HTINF part (A) in appendix}]
By definition, for any $t\in [T]$, one-hot $y\in \triangle_{[K]}$ and $x_t\in \triangle_{[K]}$, we have
\begin{equation*}
    \eta_t^{-1}-\eta_{t-1}^{-1}=\sigma \left (t^{\nicefrac 1\alpha}-(t-1)^{\nicefrac 1\alpha}\right )\overset{\text{(a)}}{\le} \sigma \frac 1\alpha (t-1)^{\nicefrac 1\alpha-1}\overset{\text{(b)}}{\le} 2\sigma \frac 1\alpha t^{\nicefrac 1\alpha-1},
\end{equation*}

where (a) comes from \cref{lem:t^q-(t-1)^q} and (b) comes from the fact that $t\ge 1$ and $\frac 1\alpha-1\ge -\frac 12$. Moreover, by definition of $\Psi(x)=-\alpha \sum_{i=1}^K x_i^{\nicefrac 1\alpha}$, we have
\begin{equation*}
    \Psi(y)-\Psi(x)=\alpha \sum_{i=1}^K x_i^{\nicefrac 1\alpha}-\alpha \sum_{i=1}^Ky_i^{\nicefrac 1\alpha}=\alpha \sum_{i=1}^K x_i^{\nicefrac 1\alpha}-\alpha\le \alpha \sum_{i\ne i^\ast}x_{t,i}^{\nicefrac 1\alpha}
\end{equation*}

from the assumption that $y$ is an one-hot vector. Therefore, we have
\begin{equation*}
    \E[\text{(A)}]=\sum_{t=1}^T [(\eta_t^{-1}-\eta_{t-1}^{-1})(\Psi(y)-\Psi(x_t))\mid \mathcal F_{t-1}]\le \E\left [\sum_{t=1}^T \sum_{i\ne i^\ast}2\sigma t^{\nicefrac 1\alpha-1}x_{t,i}^{\nicefrac 1\alpha}\right ],
\end{equation*}

which further implies (by \cref{lem:x^t to K^{1-t}})
\begin{equation*}
    \E[\text{(A)}]=\sum_{t=1}^T [(\eta_t^{-1}-\eta_{t-1}^{-1})(\Psi(y)-\Psi(x_t))\mid \mathcal F_{t-1}]\le \sum_{t=1}^T2\sigma t^{\nicefrac 1\alpha-1}K^{1-\nicefrac 1\alpha}.
\end{equation*}
\end{proof}

\begin{proof}[Proof of \cref{lem:HTINF part (B)  in appendix} (and also \cref{lem:part-B-expectation-on-tilde})]
Consider a summand before taking expectation, i.e., $\eta_t^{-1}D_\Psi(x_t, z_t)$. Let $f(x)=-\alpha x^{\nicefrac 1 \alpha}$, we then have
\begin{align*}
    \eta_t^{-1} D_\Psi(x_t, z_t) & \stackrel{\text(a)}{=} \eta_t^{-1} D_{\Psi^*}(\nabla\Psi(z_t), \nabla\Psi(x_t))
    \\
    & = \Psi^*(\nabla\Psi(z_t)) - \Psi^*(\nabla\Psi(x_t)) - \langle  x_t, \nabla\Psi(z_t) - \nabla\Psi(x_t)\rangle \\
    & \stackrel{\text(b)}{\le} \eta_t^{-1} \sum_{i=1}^K \frac 1 2 \max\{f''(x_{t,i})^{-1}, f''(z_{t,i})^{-1}\}\cdot \eta_t^2(\hat \ell_{t,i} - \ell_{t,i_t})^2\\
    & \le \eta_t^{-1} \sum_{i=1}^K \frac{\alpha}{2(\alpha - 1)} \max\{x_{t,i}, z_{t,i}\}^{2-\nicefrac 1\alpha} \eta_t^2(\hat \ell_{t,i} - \ell_{t,i_t})^2 \\
    & \le \eta_t^{-1} \sum_{i=1}^K \frac{\alpha}{2(\alpha - 1)} (2^{\frac \alpha {2\alpha - 1}})^{2-\nicefrac 1\alpha}x_{t,i}^{2-\nicefrac 1\alpha} \eta_t^2(\hat \ell_{t,i} - \ell_{t,i_t})^2 \\
    & = \frac{\alpha}{\alpha - 1}\eta_t\sum_{i=1}^K x_{t,i}^{2-\nicefrac{1}{\alpha}} (\hat \ell_{t,i} - \ell_{t,i_t})^2 \\
    & = \frac{\alpha}{\alpha - 1}\eta_t \ell_{t,i_t}^2\sum_{i=1}^K x_{t,i}^{2-\nicefrac{1}{\alpha}} \left (1 - \frac {\mathbbm 1[i_t = i]} {x_{t,i_t}}\right )^2 \\
    & \le \frac{\alpha}{\alpha - 1}\eta_t r_t^{2-\alpha} \lvert\ell_{t,i_t}\rvert^\alpha\sum_{i=1}^K x_{t,i}^{2-\nicefrac{1}{\alpha}} \left (1 - \frac {\mathbbm 1[i_t = i]} {x_{t,i_t}}\right )^2 \\
    & \stackrel{\text(c)}{\le} \frac{\alpha}{\alpha - 1}t^{\nicefrac{1}{\alpha}-1}\sigma^{1 - \alpha}\Theta_\alpha^{2-\alpha} \lvert\ell_{t,i_t}\rvert^\alpha x_{t,i_t}^{\nicefrac{2}{\alpha} - 1}\sum_{i=1}^K x_{t,i}^{2-\nicefrac{1}{\alpha}} \left (1 - \frac {\mathbbm 1[i_t = i]} {x_{t,i_t}}\right )^2 \\
    & \stackrel{\text(d)}{\le} 2t^{\nicefrac{1}{\alpha}-1}\sigma^{1 - \alpha}\lvert\ell_{t,i_t}\rvert^\alpha x_{t,i_t}^{\nicefrac{2}{\alpha} - 1}\sum_{i=1}^K x_{t,i}^{2-\nicefrac{1}{\alpha}} \left (1 - \frac {\mathbbm 1[i_t = i]} {x_{t,i_t}}\right )^2
\end{align*}
where step (a) is due to the duality property of Bregman divergences, step (b) regards $D_{\Psi^*}(\cdot, \cdot)$ as a second-order Lagrange remainder. step (c) plugs in $\eta_t^{-1} = \sigma t^{\nicefrac 1\alpha}$ and $r_t=\Theta_\alpha \eta_t^{-1}x_{t,i_t}^{\nicefrac 1 \alpha}$, thus $\eta_t r_t^{2-\alpha} = t^{\nicefrac 1 \alpha -1}\sigma^{1-\alpha}\Theta_\alpha^{2-\alpha}x_{t,i_t}^{\nicefrac 2 \alpha - 1}$. Step (d) uses $\Theta_\alpha \le (2 - \frac 2 \alpha)^{\frac 1 {2 - \alpha}}$ and thus $\Theta_\alpha^{2-\alpha} \le 2\cdot \frac{\alpha - 1} \alpha$.

After taking expectations, we get
\begin{align*}
    \E\left[\eta_t^{-1}D_\Psi(x_t, z_t) \mid \mathcal F_{t-1}\right]
    & \le 2t^{\nicefrac{1}{\alpha}-1}\sigma \sum_{i=1}^Kx_{t,i}^{\nicefrac{2}{\alpha}}\left[ \underbrace{\sum_{j=1}^K x_{t,j}^{2-\nicefrac{1}{\alpha}}}_{\le 1\le x_{t,i}^{-\nicefrac 1\alpha}} - 2x_{t,i}^{1 - \nicefrac{1}{\alpha}} + x_{t,i}^{-\nicefrac{1}{\alpha}} \right] \nonumber \\
    & \le 2\sigma t^{\nicefrac{1}{\alpha}-1}\cdot 2\left[ - \sum_{i=1}^K x_{t,i}^{1 + \nicefrac{1}{\alpha}} + \sum_{i=1}^K x_{t,i}^{\nicefrac{1}{\alpha}} \right] \nonumber \\
    & = 4\sigma t^{\nicefrac{1}{\alpha}-1}\sum_{i=1}^K x_{t,i}^{\nicefrac{1}{\alpha}}(1 - x_{t,i}) \nonumber \\
    & \le 8\sigma t^{\nicefrac{1}{\alpha}-1}\sum_{i\ne i^\ast} x_{t,i}^{\nicefrac{1}{\alpha}},
\end{align*}

where the last step is due to the fact that $1-x_{t,i^\ast}=\sum_{i\ne i^\ast}x_{t,i}\le \sum_{i\ne i^\ast}x_{t,i}^{\nicefrac 1\alpha}$ and $1-x_{t,i}\le 1$ for any $i\ne i^\ast$.
After applying \cref{lem:x^t to K^{1-t}}, we get
\begin{equation*}
     \E\left[\eta_t^{-1}D_\Psi(x_t, z_t) \mid \mathcal F_{t-1}\right] \le 8\sigma t^{\nicefrac{1}{\alpha}-1}K^{1 - \nicefrac 1 \alpha}.
\end{equation*}

Hence, we have
\begin{align*}
    \E\left [\sum_{t=1}^T\eta_{t}^{-1}D_\Psi(x_t,z_t)\right ]=\sum_{t=1}^T\E\left [\eta_{t}^{-1}D_\Psi(x_t,z_t)\mid \mathcal F_{t-1}\right ]&\le \E\left [\sum_{t=1}^T \sum_{i\ne i^\ast}8\sigma t^{\nicefrac 1\alpha-1}x_{t,i}^{\nicefrac 1\alpha}\right ]\\&\le \sum_{t=1}^T 8\sigma t^{\nicefrac 1\alpha-1}K^{1-\nicefrac 1\alpha}.
\end{align*}
\end{proof}

\section{Formal Analysis of \texttt{OptTINF} (\cref{alg-AdaHTINF})}\label{sec:formal proof of AdaHTINF}
\subsection{Main Theorem}
In this section, we present a formal proof of \cref{thm:AdaHTINF main theorem}. We still state a regret guarantee without any big-Oh notation first.

\begin{theorem}[Regret Guarantee of \cref{alg-AdaHTINF}]
If Assumptions \ref{assump:heavy-tail} and \ref{assump:truncated non-negative} hold, \cref{alg-AdaHTINF} enjoys:
\begin{enumerate}
    \item For adversarial environments, the regret is bounded by
    \begin{equation*} \mathcal R_T\le 26\sigma^\alpha K^{\frac{\alpha - 1} 2} (T+1)^{\frac {3-\alpha} 2} + 4 \sqrt{K(T+1)}.
    \end{equation*}
    \item Moreover, if the environment is stochastically constrained with a unique best arm $i^\ast$ (\cref{assump:unique best arm}), then \cref{alg-AdaHTINF} enjoys
    \begin{align*}
        \mathcal R_T&\le 2\times  4^{\frac{3-\alpha}{\alpha-1}}5^{\frac{2}{\alpha-1}}\sigma^{\frac{2\alpha}{\alpha-1}}\sum_{i\ne i^\ast}\Delta_i^{\frac{\alpha-3}{\alpha-1}}\ln(T+1) \\&+\frac{32\sigma}{\alpha-1} \sum_{i\ne i^\ast}\Delta_i^{-1}\ln(T+1)\\&+2\times  8^{\frac{2}{\alpha-1}}4^{\frac{3-\alpha}{\alpha-1}}\sigma^{\frac{2\alpha}{\alpha-1}}\sum_{i\ne i^\ast}\Delta_i^{\frac{\alpha-3}{\alpha-1}}\ln(T+1).
    \end{align*}
\end{enumerate}
\end{theorem}
\begin{proof}
In \cref{alg-AdaHTINF}, when the parameters are set as $\alpha=2$ and $\sigma=1$, we have $\eta_t^{-1}=\sqrt t$ and $r_t=\Theta_2 \sqrt{t}\sqrt{x_{t,i_t}}$ where $\Theta_2=1-2^{-\nicefrac 13}$ is an absolute constant. From now on, to avoid confusion, we use $\alpha,\sigma$ only to denote the real (hidden) parameters of the environment, instead of the parameters of the algorithm.

Following the proof of \cref{thm:HTINF main theorem} in \cref{sec:formal proof of HTINF}, we still decompose $\mathcal R_T(y)$ for $y=\mathbf e_{i^\ast}$ into $\mathcal R_T^s$ and $\mathcal R_T^f$, as follows.

\begin{equation*}
    \mathcal R_T(y)\triangleq \sum_{t=1}^T \mathbb E[\langle x_t-y,\mu_t\rangle]=\mathbb E\left [\sum_{t=1}^T\langle x_t-y,\mu_t - \mu'_t\rangle\right ]+\mathbb E\left [\sum_{t=1}^T\langle x_t-y,\mu'_t\rangle\right ]\triangleq \mathcal R_T^s+\mathcal R_T^f.
\end{equation*}

Following the analysis of \cref{alg-HTINF}, we have the following lemma.
\begin{lemma}\label{lem:AdaHTINF skipped slots in CSA}
For the given $y=\mathbf e_{i^\ast}\in \triangle_{[K]}$, \cref{alg-AdaHTINF} ensures
\begin{equation*}
    \mathcal R_{\mathcal T}^s\le \E\left [5\sigma^\alpha \sum_{t=1}^T \sum_{i\ne i^\ast}t^{\nicefrac 12-\nicefrac \alpha 2}x_{t,i}^{\nicefrac{(3-\alpha)}{2}}\right ],
\end{equation*}

which further implies
\begin{equation*}
    \mathcal R_{\mathcal T}^s\le 5\sigma^\alpha \sum_{t=1}^T t^{\frac{1-\alpha}{2}}K^{\frac{\alpha-1}{2}}.
\end{equation*}
\end{lemma}

We continue our analysis by bounding the FTRL part, $\mathcal R_T^f$. As in \cref{sec:formal proof of HTINF}, we also decompose it into two parts from \cref{lem:FTRL regret decomposition}:
\begin{equation*}
    \sum_{t=1}^T \langle x_t-y,\hat \ell_t\rangle\le \underbrace{\sum_{t=1}^T \left (\eta_t^{-1}-\eta_{t-1}^{-1}\right )(\Psi(y)-\Psi(x_t))}_{\text{Part (A)}}+\underbrace{\sum_{t=1}^T \eta_t^{-1}D_\Psi(x_t, z_t)}_{\text{Part (B)}},
\end{equation*}

where $z_t\triangleq \nabla \Psi^\ast(\nabla \Psi(x_t)-\eta_t \mathbbm 1[\lvert \ell_{t,i_t}\rvert \le r_t](\hat \ell_t- \ell_{t,i_t}\mathbf 1))$. For Part (A), from \cref{lem:HTINF part (A) in appendix}, we have (recall that $\Psi$ is now $\frac 12$-Tsallis entropy)
\begin{lemma}\label{lem:AdaHTINF part (A) in appendix}
For part (A), for any one-hot vector $y\in \triangle_{[K]}$, \cref{alg-AdaHTINF} ensures
\begin{equation*}
    \E[\text{(A)}]=\E\left[\sum_{t=1}^T \E[(\eta_t^{-1}-\eta_{t-1}^{-1})(\Psi(y)-\Psi(x_t))\mid \mathcal F_{t-1}]\right]\le \E\left [\sum_{t=1}^T 2t^{-\nicefrac 12}\sum_{i\ne i^\ast} x_{t,i}^{\nicefrac 12} \right ],
\end{equation*}

which further implies
\begin{equation*}
    \E[\text{(A)}]=\E\left[\sum_{t=1}^T \E[(\eta_t^{-1}-\eta_{t-1}^{-1})(\Psi(y)-\Psi(x_t))\mid \mathcal F_{t-1}]\right]\le \sum_{t=1}^T 2 t^{-\nicefrac 12}K^{\nicefrac 12}.
\end{equation*}
\end{lemma}

For part (B), we have
\begin{lemma}\label{lem:AdaHTINF part (B) in appendix}
For part (B), \cref{alg-AdaHTINF} ensures
\begin{equation*}
    \E[\text{(B)}]=\E\left[\sum_{t=1}^T \E[\eta_t^{-1}D_\Psi(x_t, z_t)\mid \mathcal F_{t-1}]\right]\le \E\left [\sum_{t=1}^T \sum_{i\ne i^\ast} 8 \Theta_2^{2-\alpha} \sigma^\alpha t^{\frac{1-\alpha}{2}}x_{t,i}^{\nicefrac{(3-\alpha)}{2}} \right ],
\end{equation*}

which further implies
\begin{equation*}
    \E[\text{(B)}]=\E\left[\sum_{t=1}^T \E[\eta_t^{-1}D_\Psi(x_t, z_t)\mid \mathcal F_{t-1}]\right]\le \sum_{t=1}^T 8 \Theta_2^{2-\alpha} \sigma^\alpha t^{\frac{1-\alpha}{2}}K^{\frac{\alpha-1}{2}}.
\end{equation*}
\end{lemma}

Therefore, for adversarial case (i.e., the first statement), we have
\begin{align*}
    \mathcal R_T &=\mathcal R_T^s+\mathcal R_T^f\le \mathcal R_T^s+\E[\text{(A)}]+\E[\text{(B)}]\\
    &\le 13\sigma^\alpha \sum_{t=1}^T t^{\frac {1-\alpha} 2} K^{\frac{\alpha - 1} 2} + 2\sum_{t=1}^T t^{-\nicefrac 1 2} K^{\nicefrac 1 2} \\
    & \le 26\sigma^\alpha K^{\frac{\alpha - 1} 2} (T+1)^{\frac {3-\alpha} 2} + 4 \sqrt{K(T+1)},
\end{align*}
where the last step uses \cref{lem:t^t to T^{t+1}}.

Moreover, for the stochastically constrained case with a unique best arm $i^\ast\in [K]$, with the help of AM-GM inequality, we bound each of $\mathcal R_T^s$, $\E[\text{(A)}]$ and $\E[\text{(B)}]$ by
\begin{align*}
    \mathcal R_T^s&\le \E\left [\sum_{t=1}^T \sum_{i\ne i^\ast}\left (5^{\frac {2}{\alpha-1}}\sigma^{\frac{2\alpha}{\alpha-1}}\left (\frac{\Delta_i}{4}\right )^{-\frac{3-\alpha}{\alpha-1}} \frac 1t\right )^{\frac{\alpha-1}{2}} \left (\frac{\Delta_i}{4}x_{t,i}\right )^{\frac{3-\alpha}{2}}\right ]\\
    &\le \frac{\alpha-1}{2} 4^{\frac{3-\alpha}{\alpha-1}}5^{\frac{2}{\alpha-1}}\sigma^{\frac{2\alpha}{\alpha-1}}\sum_{i\ne i^\ast}\Delta_i^{\frac{\alpha-3}{\alpha-1}}\ln (T+1)+\frac{3-\alpha}{2}\frac{\mathcal R_T}{4},\\
    \E[\text{(A)}]&\le \E\left [\sum_{t=1}^T\sum_{i\ne i^\ast} \left (4 \sigma \left (\frac{\Delta_i}{4}\right )^{-1} \frac 1t\right )^{\nicefrac 12}\left (\frac{\Delta_i}{4}x_{t,i}\right )^{\nicefrac 12}\right ]\\
    &\le \frac 12\cdot 16\sigma \sum_{i\ne i^\ast}\Delta_i^{-1}\ln (T+1) +\frac 12\frac{\mathcal R_T}{4},\\
    \E[\text{(B)}]&\le \E\left [\sum_{t=1}^T \sum_{i\ne i^\ast}\left (8^{\frac{2}{\alpha-1}}\sigma^{\frac{2\alpha}{\alpha-1}}\left (\frac{\Delta_i}{4}\right )^{-\frac{3-\alpha}{\alpha-1}} \frac 1t\right )^{\frac{\alpha-1}{2}} \left (\frac{\Delta_i}{4}x_{t,i}\right )^{\frac{3-\alpha}{2}}\right ]\\
    &\le \frac{\alpha-1}{2} 8^{\frac{2}{\alpha-1}}4^{\frac{3-\alpha}{\alpha-1}}\sigma^{\frac{2\alpha}{\alpha-1}}\sum_{i\ne i^\ast}\Delta_i^{\frac{\alpha-3}{\alpha-1}}\ln (T+1) +\frac{3-\alpha}{2}\frac{\mathcal R_T}{4}.
\end{align*}

Therefore, we have
\begin{align*}
    (1-\frac{(2-\alpha)+1+(3-\alpha)}{2}\frac 14)\mathcal R_T=\frac{\alpha-1}{4}\mathcal R_T&\le \frac{\alpha-1}{2} 4^{\frac{3-\alpha}{\alpha-1}}5^{\frac{2}{\alpha-1}}\sigma^{\frac{2\alpha}{\alpha-1}}\sum_{i\ne i^\ast}\Delta_i^{\frac{\alpha-3}{\alpha-1}}\ln (T+1) \\&+\frac 12\cdot 16\sigma \sum_{i\ne i^\ast}\Delta_i^{-1}\ln ( T+1)\\&+\frac{\alpha-1}{2} 8^{\frac{2}{\alpha-1}}4^{\frac{3-\alpha}{\alpha-1}}\sigma^{\frac{2\alpha}{\alpha-1}}\sum_{i\ne i^\ast}\Delta_i^{\frac{\alpha-3}{\alpha-1}}\ln (T+1),
\end{align*}

which gives our result.
\end{proof}

\begin{proof}[Proof of \cref{thm:AdaHTINF main theorem}]
It is a direct consequence of the theorem above.
\end{proof}

\subsection{Proof when Bounding $\mathcal R_T^s$ (\textit{the skipped part})}
\begin{proof}[Proof of \cref{lem:AdaHTINF skipped slots in CSA}]
For any $t\in [T]$ and $i\in [K]$, we can bound between the difference between the loss mean, $\mu_{t,i}$, and the truncated loss mean, $\mu_{t,i}'$, as
\begin{align*}
    \mu_{t,i}-\mu_{t,i}' &=\E[\ell_{t,i}\mathbbm 1[\lvert \ell_{t,i}>r_t\rvert]\mid \mathcal F_{t-1},i_t=i]\le \E[\lvert \ell_{t,i}\rvert^\alpha r_t^{1-\alpha}\cdot \mathbbm 1[\lvert \ell_{t,i}\rvert>r_t]\mid \mathcal F_{t-1},i_t=i]\\
    &\le \E[\lvert \ell_{t,i}\rvert^\alpha r_t^{1-\alpha}\mid \mathcal F_{t-1},i_t=i]\le \sigma^\alpha \Theta_2^{1-\alpha} t^{\frac{1-\alpha}{2}}x_{t,i}^{\frac{1-\alpha}{2}}.
\end{align*}

Hence, we have
\begin{align*}
    \mathcal R_T^s=\sum_{t=1}^T \E\left [\langle x_t-y,\mu_t-\mu_t'\rangle\right ] & \le \E\left [\sigma^\alpha \Theta_2^{1-\alpha} \sum_{t=1}^T \sum_{i\ne i^\ast}t^{\nicefrac 12-\nicefrac \alpha 2}x_{t,i}^{\nicefrac 12-\nicefrac \alpha 2}\cdot x_{t,i}\right ]\\
    & \le \E\left [5\sigma^\alpha  \sum_{t=1}^T \sum_{i\ne i^\ast}t^{\nicefrac 12-\nicefrac \alpha 2}x_{t,i}^{\nicefrac 32-\nicefrac \alpha 2}\right ],
\end{align*}
where the last step uses $\Theta_2^{1-\alpha} \le \Theta_2^{-1}\ \le 5$. It further gives, by \cref{lem:x^t to K^{1-t}}, that
\begin{equation*}
    \mathcal R_T^s\le 5\sigma^\alpha \sum_{t=1}^T t^{\nicefrac 12-\nicefrac \alpha 2}K^{\nicefrac \alpha 2-\nicefrac 12}.
\end{equation*}
\end{proof}

\subsection{Proof when Bounding $\mathcal R_T^f$ (\textit{the FTRL part})}
\begin{proof}[Proof of \cref{lem:AdaHTINF part (A) in appendix}]
This is just a restatement of \cref{lem:HTINF part (A) in appendix}.
\end{proof}
\begin{proof}[Proof of \cref{lem:AdaHTINF part (B) in appendix}]
We simply follow the proof of \cref{lem:HTINF part (B) in appendix}, except for some slight modifications (instead of the previous lemma, we cannot directly modify all $\alpha$'s to $2$, as the second moment of $\ell_{t,i_t}$ may not exist). The first few steps are exactly the same, giving
\begin{align*}
    \eta_t^{-1}D_\Psi(x_t,z_t)&\le \frac{2}{2-1}\eta_t r_t^{2-\alpha}\lvert\ell_{t,i_t}\rvert^\alpha\sum_{i=1}^K x_{t,i}^{2-\nicefrac 12}\left (1-\frac{\mathbbm 1[i_t=i]}{x_{t,i_t}}\right )^2\\&\le 2 \left (t^{\nicefrac 12}\right )^{-1} \Theta_2^{2-\alpha} \left (t^{\nicefrac 12}\right )^{2-\alpha} x_{t,i_t}^{\frac{2-\alpha}{2}} \lvert \ell_{t,i_t} \rvert^\alpha \sum_{i=1}^K x_{t,i}^{2-\nicefrac 12}\left (1-\frac{\mathbbm 1[i_t=i]}{x_{t,i_t}}\right )^2\\
    &=2\Theta_2^{2-\alpha} \lvert \ell_{t,i_t} \rvert^\alpha t^{\frac{1-\alpha}{2}} x_{t,i_t}^{\frac{2-\alpha}{2}} \sum_{i=1}^K x_{t,i}^{2-\nicefrac 12}\left (1-\frac{\mathbbm 1[i_t=i]}{x_{t,i_t}}\right )^2.
\end{align*}

After taking expectations, we have
\begin{align*}
    \E[\eta_t^{-1}D_\Psi(x_t,z_t)\mid \mathcal F_{t-1}]&\le 2\Theta_2^{2-\alpha} \sigma^\alpha t^{\frac{1-\alpha}{2}} \sum_{i=1}^K x_{t,i}^{2-\nicefrac \alpha 2} \left [\underbrace{\sum_{j=1}^K x_{t,j}^{\nicefrac 32}}_{\le 1\le x_{t,i}^{-\nicefrac 12}}-2x_{t,i}^{\nicefrac 12}+x_{t,i}^{-\nicefrac 12}\right ]\\
    &\le 4\Theta_2^{2-\alpha} \sigma^\alpha t^{\frac{1-\alpha}{2}} \left [\sum_{i=1}^K x_{t,i}^{\nicefrac 32-\nicefrac \alpha 2}-\sum_{i=1}^K x_{t,i}^{\nicefrac 52-\nicefrac \alpha 2}\right ]\\
    &= 4\Theta_2^{2-\alpha} \sigma^\alpha t^{\frac{1-\alpha}{2}} \left [\sum_{i=1}^K x_{t,i}^{\nicefrac 32-\nicefrac \alpha 2}(1-x_{t,i})\right ]\\
    &\le 8\Theta_2^{2-\alpha} \sigma^\alpha t^{\frac{1-\alpha}{2}} \sum_{i\ne i^\ast} x_{t,i}^{\nicefrac 32-\nicefrac \alpha 2}
\end{align*}

Therefore, we have
\begin{equation*}
    \E\left [\sum_{t=1}^T \eta_t^{-1}D_\Psi(x_t,z_t)\right ]\le \E\left [\sum_{t=1}^T \sum_{i\ne i^\ast} 8 \Theta_2^{2-\alpha} \sigma^\alpha t^{\frac{1-\alpha}{2}}x_{t,i}^{\nicefrac{(3-\alpha)}{2}} \right ],
\end{equation*}

which further gives
\begin{equation*}
    \E\left [\sum_{t=1}^T \eta_t^{-1}D_\Psi(x_t,z_t)\right ]\le \sum_{t=1}^T 8 \Theta_2^{2-\alpha} \sigma^\alpha t^{\frac{1-\alpha}{2}}K^{\frac{\alpha-1}{2}}
\end{equation*}

by \cref{lem:x^t to K^{1-t}}.
\end{proof}

\section{Formal Analysis of \texttt{AdaTINF}  (\cref{alg-AdaTINF})}\label{sec:formal proof of AdaTINF}
\subsection{Main Theorem}
We again begin with a regret guarantee without any big-Oh notations.

\begin{theorem}[Regret Guarantee of \cref{alg-AdaTINF}]
If Assumptions \ref{assump:heavy-tail} and \ref{assump:truncated non-negative} hold, \cref{alg-AdaTINF} ensures
\begin{equation*}
    \mathcal R_T \le 3\sqrt{K(T+1)} + 204\sigma K^{1-\nicefrac 1 \alpha} (T+1)^{\nicefrac 1 \alpha} + 12\sigma T^{\nicefrac 1 \alpha}.
\end{equation*}
\end{theorem}
\begin{proof}
As defined in the text, we group time slots with equal $\lambda_t$'s into epochs, as
\begin{equation*}
    \mathcal T_j\triangleq \{t\in [T]\mid \lambda_t=2^j\},\quad \forall j\ge 0.
\end{equation*}

For any non-empty $\mathcal T_j$'s, denote the first and last time slot of $\mathcal T_j$ by
\begin{equation*}
    \gamma_j\triangleq \min\{t\in \mathcal T_j\},\tau_j\triangleq \max\{t\in \mathcal T_j\}.
\end{equation*}

Then, without loss of generality, assume that no doubling has happened for time slot $T$. Otherwise, one can always add a virtual time slot $t=T+1$ with $\ell_{t,i}=0$ for all $i$ . Therefore, we have $\mathcal T_J\ne \varnothing$ where $J$ is the final value of variable $J$ defined in the code.

We adopt the notation of $c_t$ as defined in \cref{alg-AdaTINF}:
\begin{equation*}
    c_t=2\eta_t x_{t,i_t}^{-\nicefrac 12}\ell_{t,i_t^2}\mathbbm 1[\lvert \ell_{t,i_t}\rvert\le r_t]+\ell_{t,i_t}\mathbbm 1[\lvert \ell_{t,i_t}\rvert>r_t].
\end{equation*}

Moreover, from the doubling criterion of \cref{alg-AdaTINF}, for each non-empty epoch, we have the following lemma.
\begin{lemma}
\label{lem:epoch-bounds-apdx}
For any $0\le j<J$ such that $\mathcal T_j\ne \varnothing$, we have
\begin{align}
    \mathbbm 1[\gamma_j>1]c_{\gamma_j-1}+\sum_{t\in \mathcal T_j\setminus \{\tau_j\}}c_t&\le 2^j\sqrt{K(T+1)},\label{eq:epoch-bound-apdx}\\
    \sum_{t\in \mathcal T_j}c_t>2^{j-1}\sqrt{K(T+1)}, \label{eq:epoch-bound-exceed-apdx}
\end{align}

Moreover, for $j=J$ (recall that $\mathcal T_J\ne \varnothing$), we have
\begin{equation}
    \mathbbm 1[\gamma_j>1]c_{\gamma_j-1}+\sum_{t\in \mathcal T_j}c_t\le 2^J\sqrt{K(T+1)}.\label{eq:epoch-bound-final-apdx}
\end{equation}
\end{lemma}

Similar to previous analysis, we define $\mu_{t,i}'=\E[\ell_{t,i}\mathbbm 1[\lvert \ell_{t,i}\rvert \le r_t]\mid \mathcal F_{t-1},i_t=i]$ and decompose the regret $\mathcal R_T(y)$ as follows
\begin{equation*}
    \mathcal R_T(y)=\E\left [\sum_{t=1}^T \langle x_t-y,\mu_t-\mu_t'\rangle\right ]+\E\left [\sum_{t=1}^T \langle x_t-y,\hat \ell_t\rangle\right ]\triangleq \mathcal R_T^s+\mathcal R_T^f.
\end{equation*}
According to \cref{lem:HTINF skipped slots}, we have
\begin{align}
    R_T^s &\le \E\left[\sum_{t=1}^T \sum_{i=1}^K x_{t,i}(\mu_{t,i} - \mu'_{t,i})\right] \nonumber \\
    & = \E\left[\sum_{t=1}^T \sum_{i=1}^K \ell_{t,i_t}\cdot\mathbbm 1[\lvert\ell_{t,i_t}\vert > r_t]\right]. \label{eq:apdx-alg3-rts}
\end{align}
Furthermore, due to the properties of weighted importance sampling estimator (as in \cref{sec:formal proof of HTINF}, $\E[\hat \ell_{t,i}\mid \mathcal F_{t-1}]=\mu_{t,i}'$), we have
\begin{equation*}
    \mathcal R_T^f=\E\left [\sum_{t=1}^T \langle x_t-y,\hat \ell_t\rangle\right ].
\end{equation*}

We can then apply \cref{lem:FTRL regret decomposition} to $\mathcal R_T^f$, giving
\begin{align*}
    \sum_{t=1}^T \langle x_t-y,\hat \ell_t\rangle & \le \eta_T \max_{x\in \triangle_{[K]}}\Psi(x)+\sum_{t=1}^T \eta_t^{-1}D_\Psi(x_t,z_t)
\end{align*}
where $z_t\triangleq \nabla \Psi^\ast(\nabla \Psi(x_t)-\eta_t\hat \ell_t)$. The first term is simply within $2^J \sqrt{KT}$. For the second term, we have the following property similar to \cref{lem:HTINF part (B) in appendix}:

\begin{lemma}
\label{lem:alg3-apdx-immediate-cost}
\cref{alg-AdaTINF} guarantees that for any $t\in [T]$,
\begin{align*}
    \eta_t^{-1}D_\Psi(x_t,z_t) \le 2\eta_t x_{t,i_t}^{\nicefrac 3 2} \hat\ell_{t,i_t}^2
\end{align*}
where $z_t\triangleq \nabla \Psi^\ast(\nabla \Psi(x_t)-\eta_t\hat \ell_t)$.
\end{lemma}

Thus we have
\begin{align}
    \mathcal R_T^f & \le \E[2^J]\sqrt{KT} + \E\left[\sum_{t=1}^T 2\eta_t x_{t,i_t}^{\nicefrac 3 2} \hat\ell_{t,i_t}^2\right] \nonumber \\
    & = \E[2^J]\sqrt{KT} + \E\left[\sum_{t=1}^T 2\eta_t x_{t,i_t}^{-\nicefrac 1 2} \ell_{t,i_t}^2 \cdot \mathbbm 1[\lvert\ell_{t,i_t}\rvert \le r_t]\right]. \label{eq:apdx-alg3-rtf}
\end{align}

Combining Eq. (\ref{eq:apdx-alg3-rtf}) and (\ref{eq:apdx-alg3-rtf}), we can see
\begin{align*}
    \mathcal R_T & \le \E[2^J]\sqrt{KT} + \E\left[\sum_{t=1}^T \left( \ell_{t,i_t}\cdot\mathbbm 1[\lvert\ell_{t,i_t}\vert > r_t] + 2\eta_t x_{t,i_t}^{-\nicefrac 1 2} \ell_{t,i_t}^2 \cdot \mathbbm 1[\lvert\ell_{t,i_t}\rvert \le r_t]\right)\right] \\
    & = \E[2^J]\sqrt{KT} + \E\left[\sum_{t=1}^T c_t\right].
\end{align*}
Summing up Equation (\ref{eq:epoch-bound-apdx}) for all non-empty epoch $j<J$ and Equation (\ref{eq:epoch-bound-final-apdx}), we get
\begin{equation*}
    \sum_{t=1}^T c_t = \sum_{j=0}^J \sum_{t\in \mathcal T_j} c_t \le \sum_{j=0}^J 2^J\sqrt{K(T+1)} \le 2^{J+1} \sqrt{K(T+1)},
\end{equation*}
and we can conclude
\begin{equation*}
    \mathcal R_T \le \E[2^J]\cdot 3\sqrt{K(T+1)}.
\end{equation*}
It remains to bound $\E[2^J]$. When $J\ge 1$, there are at least two non-empty epochs. Let $J'$ be the index of the second last epoch. The doubling condition of \cref{alg-AdaTINF} further reduce the task to bound $2^J$ into bounding $2^{J'}$ and $c_{\tau_{J'}}$, as the following lemma states.

\begin{lemma}
\label{lem:2^J bound in 2^J' and c_tau_J'}
\cref{alg-AdaTINF} guarantees that, when $J\ge 1$, we have
\begin{equation}
    2^J\sqrt{K(T+1)}\le 2^{J'+1}\sqrt{K(T+1)}+4c_{\tau_{J'}}. \label{eq: 2^J bound in 2^J' and c_tau_J'}
\end{equation}
\end{lemma}

We can derive the following expectation bound for both $2^{J'}$ and $c_{\tau_{J'}}$:

\begin{lemma}[Restatement of \cref{lem:2^J'}]
\label{lem:2^J'-apdx}
\cref{alg-AdaTINF} guarantees that
\begin{equation}
\E[\mathbbm 1[J\ge 1]2^{J'}]\le 28\sigma K^{\nicefrac 12-\nicefrac 1\alpha}(T+1)^{\nicefrac 1\alpha-\nicefrac 12}. \label{eq:2^J' bound}
\end{equation}
\end{lemma}

\begin{lemma}[Restatement of \cref{lem:c_tau_J'}]
\label{lem:c_tau_J'-apdx}
\cref{alg-AdaTINF} guarantees that 
\begin{equation}
\E[\mathbbm 1[J\ge 1]c_{\tau_{J'}}]\le 0.1\E[\mathbbm 1[J\ge 1]2^{J'}\sqrt T]+\E\left [\max\limits_{t\in [T]}\lvert \ell_{t,i_t}\rvert\right ]. \label{eq:c_tau_J' bound}
\end{equation}
\end{lemma}

Applying \cref{lem:expectation of maximum of n variables} and Equation (\ref{eq:2^J' bound}) to Eqation (\ref{eq:c_tau_J' bound}), we get
\begin{equation*}
    \E[\mathbbm 1[J\ge 1]c_{\tau_{J'}}]\le 3\sigma K^{\nicefrac 12-\nicefrac 1\alpha}(T+1)^{\nicefrac 1\alpha} + \sigma T^{\nicefrac 1 \alpha}.
\end{equation*}
Plugging this into Equation (\ref{eq: 2^J bound in 2^J' and c_tau_J'}), we get
\begin{equation*}
    \E\left[\mathbbm 1[J\ge 1] 2^J \sqrt{K(T+1)}\right] \le 68 \sigma K^{1-\nicefrac 1\alpha}(T+1)^{\nicefrac 1\alpha} + 4\sigma T^{\nicefrac 1 \alpha},
\end{equation*}
and thus
\begin{equation*}
    \E\left[2^J \sqrt{K(T+1)}\right] \le 68 \sigma K^{1-\nicefrac 1\alpha}(T+1)^{\nicefrac 1\alpha} + 4\sigma T^{\nicefrac 1 \alpha} + \sqrt{K(T+1)},
\end{equation*}
\begin{equation*}
    \mathcal R_T \le 3\E\left[2^J \sqrt{K(T+1)}\right] \le 204 \sigma K^{1-\nicefrac 1\alpha}(T+1)^{\nicefrac 1\alpha} + 12\sigma T^{\nicefrac 1 \alpha} + 3\sqrt{K(T+1)}.
\end{equation*}

\end{proof}

\begin{proof}[Proof of \cref{thm:AdaTINF main theorem}]
It is a direct consequence of the theorem above.
\end{proof}

\subsection{Proof when Reducing $\mathcal R_T$ to $\E[2^J]$}
\begin{proof}[Proof of \cref{lem:epoch-bounds-apdx}]
It suffices to notice that in \cref{alg-AdaTINF}, during a particular epoch $j$, when the doubling condition at Line 15 evaluates to true, the current value of the variable $S_j$ is $1[\gamma_j>1]c_{\gamma_j-1}+\sum_{t\in \mathcal T_j}c_t$, thus
\begin{equation*}
    1[\gamma_j>1]c_{\gamma_j-1}+\sum_{t\in \mathcal T_j}c_t > 2^j \sqrt{K(T+1)}.
\end{equation*}
When $\gamma_j = 1$ (or equivalently, $j=0$), Equation (\ref{eq:epoch-bound-exceed-apdx}) automatically holds. Otherwise Line 16 guarantees that $j \ge \lceil\log_2 (c_{\tau_{\gamma_j - 1}}/\sqrt{K(T+1)})\rceil + 1$, hence $2^{j-1}\sqrt{K(T+1)} \ge c_{\tau_{\gamma_j - 1}}$. We will have
\begin{equation*}
    2^{j-1}\sqrt{K(T+1)} + \sum_{t\in \mathcal T_j}c_t > 2^j \sqrt{K(T+1)},
\end{equation*}
which also solves to Equation (\ref{eq:epoch-bound-exceed-apdx}).

When the doubling condition at Line 15 evaluates to false for the last time, the value of $S_j$ is $1[\gamma_j>1]c_{\gamma_j-1}+\sum_{t\in \mathcal T_j\setminus \{\tau_j\}}c_t$. At this time we have $S_j \le 2^j \sqrt{K(T+1)}$, hence Equation (\ref{eq:epoch-bound-apdx}) and (\ref{eq:epoch-bound-final-apdx}) hold.
\end{proof}

\begin{proof}[Proof of \cref{lem:alg3-apdx-immediate-cost}]

It is exactly the same calculation we did in \cref{lem:HTINF part (B) in appendix}, the only difference is that $\hat \ell_t$ does not come with a $-\ell_{t,i_t}$ drift.
\end{proof}

\subsection{Proof when Bounding $\E[2^J]$}
\begin{proof}[Proof of \cref{lem:2^J bound in 2^J' and c_tau_J'}]
According to Line 16 of \cref{alg-AdaTINF}, $J, J'$ and $c_{\tau_{J'}}$ satisfy
\begin{equation*}
    J = \max\left\{ J' + 1, \lceil\log_2 (c_{\tau_{J'}}/\sqrt{K(T+1)})\rceil + 1 \right\},
\end{equation*}
thus
\begin{equation*}
    2^J \le \max\left\{2\cdot 2^{J'}, 4\cdot c_{\tau_{J'}}/\sqrt{K(T+1)}\right\}
\end{equation*}
and
\begin{align*}
    2^J\sqrt{K(T+1)} & \le \max\left\{2\cdot 2^{J'} \sqrt{K(T+1)}, 4\cdot c_{\tau_{J'}}\right\} \\
    & \le 2\cdot 2^{J'} \sqrt{K(T+1)} + 4\cdot c_{\tau_{J'}}.
\end{align*}
\end{proof}

\begin{proof}[Proof of \cref{lem:2^J'-apdx}]
Applying Eq. (\ref{eq:epoch-bound-exceed-apdx}) to $j=J'<J$, we get
\begin{equation}
    \hspace*{-0.25cm}\mathbbm 1[J\ge 1](2^{J'})^\alpha \sqrt{K(T+1)} \le (2^{J'})^{\alpha-1}2\sum_{t\in \mathcal T_{J'}} c_t \label{eq:formula-to-solve-E-2^J}
\end{equation}
We further upper-bound the RHS of (\ref{eq:formula-to-solve-E-2^J}) by enlarging the summation range to $[T]$. Specifically, let $\tilde \eta_t = 2^{-J'}t^{-1/2}$, $\tilde r_t = 2^{J'}\Theta_2\sqrt{tx_{t,i_t}}$. Define the summands by
\begin{align}
    \tilde c_t & \triangleq 2\tilde \eta_t x_{t,i_t}^{-\nicefrac 12}\ell_{t,i_t}^2\mathbbm 1[\lvert\ell_{t,i_t}\rvert
     \le \tilde r_t] + \ell_{t,i_t}\mathbbm 1[\lvert\ell_{t,i_t}\rvert > \tilde r_t] \label{eq:tilde-c-t} \\
     & \le 2\tilde \eta_t x_{t,i_t}^{-\nicefrac 12} \lvert\ell_{t,i_t}\rvert^\alpha \tilde r_t^{2-\alpha} + \lvert\ell_{t,i_t}\rvert^\alpha \tilde r_t^{1-\alpha}\nonumber \\
     & \le (2\tilde \eta_t \tilde r_t^{2-\alpha} x_{t,i_t}^{-\nicefrac 12} + \tilde r_t^{1-\alpha} )\cdot \lvert\ell_{t,i_t}\rvert^\alpha \nonumber \\
     & = (2\Theta_2^{2-\alpha} + \Theta_2^{1 - \alpha}) \cdot 2^{(1-\alpha)J'} t^{\frac{1-\alpha} 2} x_{t,i_t}^{\frac{1-\alpha} 2}\lvert\ell_{t,i_t}\rvert^\alpha\nonumber\\
     & \le (2 + \Theta_2^{-1}) \cdot 2^{(1-\alpha)J'} t^{\frac{1-\alpha} 2} x_{t,i_t}^{\frac{1-\alpha} 2}\lvert\ell_{t,i_t}\rvert^\alpha\nonumber \\
     & \le 7\cdot 2^{(1-\alpha)J'} t^{\frac{1-\alpha} 2} x_{t,i_t}^{\frac{1-\alpha} 2}\lvert\ell_{t,i_t}\rvert^\alpha\nonumber.
\end{align}
We see that the definition in Eq. (\ref{eq:tilde-c-t}) coincides with $c_t$ for $t\in\mathcal T_{J'}$. Thus, the RHS of (\ref{eq:formula-to-solve-E-2^J}) is no more than
\begin{equation*}
    14\sum_{t=1}^T t^{\frac{1-\alpha} 2} x_{t,i_t}^{\frac{1-\alpha} 2}\lvert\ell_{t,i_t}\rvert^\alpha
\end{equation*}

Taking expectation on both sides of (\ref{eq:formula-to-solve-E-2^J}), we get
\begin{equation*}
    \E\left [\mathbbm 1[J\ge 1](2^{J'})^\alpha\right ] \sqrt{K(T+1)} \le 28\sigma^\alpha K^{\frac{\alpha - 1} 2}(T+1)^{\frac {3-\alpha} 2},
\end{equation*}
which gives $\E[\mathbbm 1[J\ge 1](2^{J'})^\alpha] \le 28\sigma^\alpha K^{\nicefrac \alpha 2 - 1} (T+1)^{1 - \nicefrac \alpha 2}$. By Jensen's inequality,
\begin{align*}
    \E\left [\mathbbm 1[J\ge 1]2^{J'}\right ] & \le \left(\E\left [\mathbbm 1[J\ge 1]^\alpha\left(2^{J'}\right)^\alpha\right ]\right)^{\nicefrac 1 \alpha}\\
    & \le 28\sigma K^{ \nicefrac 1 2 - \nicefrac 1 \alpha} (T+1)^{\nicefrac 1 \alpha - \nicefrac 1 2}.
\end{align*}
\end{proof}

\begin{proof}[Proof of \cref{lem:c_tau_J'-apdx}]
We can do the calculation
\begin{align*}
    c_{\tau_{J'}} & = 2\eta_{\tau_{J'}} x_{\tau_{J'},i_{\tau_{J'}}}^{-\nicefrac 12}\ell_{\tau_{J'},i_{\tau_{J'}}}^2\mathbbm 1[\lvert\ell_{\tau_{J'},i_{\tau_{J'}}}\rvert\le r_{\tau_{J'}}]  + \ell_{\tau_{J'},i_{\tau_{J'}}}\mathbbm 1[\lvert\ell_{\tau_{J'},i_{\tau_{J'}}}\rvert > r_{\tau_{J'}}] \\
     & \le 2\eta_{\tau_{J'
    }} x_{\tau_{J'},i_{\tau_{J'}}}^{-\nicefrac 12} r_{\tau_{J'
    }}^2 + \max_{t\in[T]} \lvert\ell_{t,i_t}\rvert \\
    & = 2^{J'}\cdot 2\Theta_2^2\sqrt{\tau_{J'
    }} x_{\tau_{J'
    },i_{\tau_{J'
    }}}^{\nicefrac 12} + \max_{t\in[T]} \lvert\ell_{t,i_t}\rvert \\
    & \le 0.1 \cdot 2^{J'} \sqrt T + \max_{t\in[T]} \lvert\ell_{t,i_t}\rvert. 
\end{align*}

\end{proof}

\section{Removing Dependency on Time Horizon $T$ in \cref{alg-AdaTINF}}\label{sec:remove T in AdaTINF}
To remove the dependency of $T$, we leverage the following doubling trick, which is commonly used for unknown $T$'s \cite{auer1995gambling,besson2018doubling}. This gives our More Adaptive \texttt{AdaTINF} algorithm, which we called $\texttt{Ada}^2\texttt{TINF}$.

\begin{algorithm}[htb]
\caption{More Adaptive \texttt{AdaTINF} ($\texttt{Ada}^2\texttt{TINF}$)}
\label{alg-AdaAdaTINF}
\begin{algorithmic}[1]
\REQUIRE{Number of arms $K$}
\ENSURE{Sequence of actions $i_1,i_2,\cdots,i_T\in [K]$}
\STATE Initialize $T_0\gets 1,S\gets 0$
\FOR{$t=1,2,\cdots$}
\IF{$t\ge S$}
\STATE{$T_0\gets 2T_0$, $S\gets S+T_0-1$}
\STATE{Initialize a new \texttt{AdaTINF} instance (\cref{alg-AdaTINF}) with parameters $K$ and $T_0-1$}
\ENDIF
\STATE{Run current \texttt{AdaTINF} instance for one time slot, act what it acts and feed it with the feedback $\ell_{t,i_t}$}
\ENDFOR
\end{algorithmic}
\end{algorithm}

\begin{theorem}[Regret Guarantee of \cref{alg-AdaAdaTINF}]
Under the same assumptions of \cref{thm:AdaTINF main theorem}, i.e., Assumptions \ref{assump:heavy-tail} and \ref{assump:truncated non-negative} hold, $\texttt{Ada}^2\texttt{TINF}$ (\cref{alg-AdaAdaTINF}) ensures
\begin{equation*}
    \mathcal R_T\le 600\sigma K^{1-\nicefrac 1\alpha}(T+1)^{\nicefrac 1\alpha}.
\end{equation*}
\end{theorem}
\begin{proof}
    We divide the time horizon $T$ into several \textit{super-epochs}, each with length $T_0-1=2^1-1,2^2-1,2^3-1,\cdots$. For each of the super-epoch, as we restarted the whole process, we can regard each of them as a independent execution of \texttt{AdaTINF}. Therefore, by \cref{thm:AdaTINF main theorem}, for an super-epoch from $t_0$ to $t_0+T_0-2$, we have
    \begin{equation*}
        \E\left [\sum_{t=t_0}^{t_0+T_0-2}\langle x_t-\mathbf e_{i^\ast},\mu_t\rangle\right ]=\mathcal R_{T_0-1}\le 300\sigma K^{1-\nicefrac 1\alpha} T_0^{\nicefrac 1\alpha}.
    \end{equation*}
    
    Therefore, the total regret is bounded by
    \begin{equation*}
        \mathcal R_{T}\le \sum_{T_0=2^1-1,2^2-1,\cdots,2^{\lceil \log_2 (T+1)\rceil}-1} 300\sigma K^{1-\nicefrac 1\alpha}(T_0+1)^{\nicefrac 1\alpha}\le 600\sigma K^{1-\nicefrac 1\alpha} T^{\nicefrac 1\alpha},
    \end{equation*}
    
    as desired.
\end{proof}

\section{Auxiliary Lemmas}
\subsection{Probability Lemmas}
\begin{lemma}\label{lem:Markov inequality for moments}
For a non-negative random variable $X$ whose $\alpha$-th moment exists and a constant $c>0$, we have
\begin{equation*}
    \Pr\{X\ge c\}\le \frac{\E[X^\alpha]}{c^\alpha}
\end{equation*}
\end{lemma}
\begin{proof}
As both $X,c$ are non-negative, $\Pr\{X\ge c\}=\Pr\{X^\alpha\ge c^\alpha\}\le \frac{\E[X^\alpha]}{c^\alpha}$ by Markov's inequality.
\end{proof}

\begin{lemma}\label{lem:bounded expectation from moment}
For a random variable $Y$ with $q$-th moment $\E[\lvert Y\rvert^q]$ bounded by $\sigma^q$ (where $q\in [1,2]$), its $p$-th moment $\E[\lvert Y\rvert^p]$ is also bounded by $\sigma^p$ if $1\le p\le q$.
\end{lemma}
\begin{proof}
As the function $f\colon x\mapsto x^\alpha$ is convex for any $\alpha\ge 1$, by Jensen's inequality, we have $f(\E[X])\le \E[f(X)]$ for any random variable $X$. Hence, by picking $X=\lvert Y\rvert^p$ and $\alpha=\frac qp$, we have $(\E[\lvert Y\rvert^p])^{\nicefrac qp}\le \E[(\lvert Y\rvert^p)^{\nicefrac qp}]=\E[\lvert Y\rvert^q]\le \sigma^q$, so $\E[\lvert Y\rvert^p]\le \sigma^p$ for any $1\le p\le q$.
\end{proof}

\begin{lemma}\label{lem:expectation of maximum of n variables}
For $n$ independent random variables $X_1,X_2,\cdots,X_n$, each with $\alpha$-th moment ($1<\alpha\le 2$) bounded by $\sigma^\alpha$, i.e., $\E_{x_i\sim X_i}[\lvert x_i\rvert^\alpha]\le \sigma^\alpha$ for all $1\le i\le n$, we have
\begin{equation*}
    \E_{x_1\sim X_1,x_2\sim X_2,\cdots,x_n\sim X_n}\left [\max_{1\le i\le n}\lvert x_i\rvert\right ]\le \sigma n^{\nicefrac 1\alpha}.
\end{equation*}
\end{lemma}
\begin{proof}
By Jensen's inequality, we have (here, $\mathbf x\sim \mathbf X$ denotes $x_1\sim X_1,x_2\sim X_2,\cdots,x_n\sim X_n$)
\begin{equation*}
    \left (\E_{\mathbf x\sim \mathbf X}\left [\max_{1\le i\le n}\lvert x_i\rvert\right ]\right )^\alpha\le \E_{\mathbf x\sim \mathbf X}\left [\left (\max_{1\le i\le n}\lvert x_i\rvert\right )^\alpha\right ]=\E_{\mathbf x\sim \mathbf X}\left [\max_{1\le i\le n}\lvert x_i\rvert^\alpha\right ]\le \E_{\mathbf x\sim \mathbf X}\left [\sum_{i=1}^n \lvert x_i\rvert^\alpha\right ]=\sum_{i=1}^n \E_{x_i\sim X_i}[\lvert x_i\rvert^\alpha]\le n\sigma^\alpha,
\end{equation*}

which gives $\E_{\mathbf x\sim \mathbf X}\left [\max_{1\le i\le n}\lvert x_i\rvert\right ]\le \sigma n^{\nicefrac 1\alpha}$.
\end{proof}

\subsection{Arithmetic Lemmas}
\begin{lemma}\label{lem:x^t to K^{1-t}}
For any $x\in \triangle_{[K]}$ (i.e., $\sum_{i=1}^K x_i=1$), we have
\begin{equation*}
    \sum_{i=1}^K x_i^{t}\le  K^{1-t}
\end{equation*}

for $\frac 12\le t<1$.
\end{lemma}
\begin{proof}
By H\"older's inequality $\lVert fg\rVert_1\le \lVert f\rVert_p\lVert g\rVert_q$, we have $\sum_{i=1}^K x_i^t\le (\sum_{i=1}^K (x_i^t)^{\nicefrac 1t})^t(\sum_{i=1}^K 1^q)^{\nicefrac 1q}=K^{1-t}$ by picking $p=\frac 1t$ and $q=\frac{1}{1-t}$.
\end{proof}

\begin{lemma}\label{lem:t^t to T^{t+1}}
For any positive integer $n$, we have
\begin{equation*}
    \sum_{i=1}^n \frac 1i\le \ln(n+1).
\end{equation*}

Moreover, for any $-1<t<0$, we have
\begin{equation*}
    \sum_{i=1}^n i^t\le \frac{(n+1)^{t+1}}{t+1}.
\end{equation*}
\end{lemma}
\begin{proof}
If $t=-1$, we have $\sum_{i=1}^n i^t\le \int_1^(n+1) \frac{\mathrm{d}x}{x}=\ln n$. If $t>-1$, we have $\sum_{i=1}^n i^t\le \int_0^{n+1} x^t~\mathrm{d}x=\frac{(n+1)^{t+1}}{t+1}$.
\end{proof}

\begin{lemma}\label{lem:t^q-(t-1)^q}
For any $x\ge 1$ and $q\in (0,1)$, we have
\begin{equation*}
    x^q-(x-1)^q\le q(x-1)^{q-1}.
\end{equation*}
\end{lemma}
\begin{proof}
Consider the function $f$ defined by $x\mapsto x^q$. We have $f''(x)=q(q-1)x^{q-2}\le 0$ for $x\ge 0$ and $q\in (0,1)$. Hence, $f(x)$ is concave for $x\ge 0$ and $q\in (0,1)$. Therefore, by properties of concave functions, we have $f(x)\le f(x-1)+f'(x-1)(x-(x-1))=f(x-1)+q(x-1)^{q-1}$ for any $x\ge 1$ and $q\in (0,1)$, which gives $x^q-x^{q-1}\le q(x-1)^{q-1}$.
\end{proof}

\subsection{Lemmas on the FTRL Framework for MAB Algorithm Design}
\begin{lemma}\label{lem:sum x Delta to regret}
For any algorithm that plays action $i_t\sim x_t$ where $\{x_t\}_{t=1}^T$ can be regarded as a stochastic process adapted to the natural filtration $\{\mathcal F_t\}_{t=0}^T$, its regret, in a stochastically constrained adversarial environment with unique best arm $i^\ast\in [K]$, is lower-bounded by
\begin{equation*}
    \mathcal R_T\ge \sum_{t\in [T]}\sum_{i\ne i^\ast}\Delta_i \E[x_{t,i}\mid \mathcal F_{t-1}].
\end{equation*}
\end{lemma}
\begin{proof}
By definition of $\mathcal R_T$ and $\Delta_i$, we have
\begin{equation*}
    \mathcal R_T=\E\left [\sum_{t=1}^T \langle x_t-\mathbf e_{i^\ast},\mu_t\rangle\right ]=\E\left [\sum_{t=1}^T \sum_{i\ne i^\ast} x_{t,i}\mu_{t,i}-(1-x_{t,i^\ast})\mu_{t,i^\ast}\right ]=\sum_{t=1}^T \E\left [\sum_{i\ne i^\ast} x_{t,i}(\mu_{t,i}-\mu_{t,i^\ast})\right ],
\end{equation*}

which is exactly $\sum_{t\in [T]}\sum_{i\ne i^\ast}\Delta_i \E[x_{t,i}\mid \mathcal F_{t-1}]$.
\end{proof}
\begin{lemma}[Property of Weighted Importance Sampling Estimator]\label{lem:importance sampler}
For any distribution $x\in \triangle_{[K]}$ and loss vector $\ell \in \mathbb R^K$ sampled from a distribution $\nu\in \triangle_{\mathbb R^k}$, if we pulled an arm $i$ according to $x$, then the weighted importance sampler $\tilde \ell(j)\triangleq \frac{\ell(j)}{x_j}\mathbbm 1[i=j]$ gives an unbiased estimate of $\E[\ell]$, i.e.,
\begin{equation*}
\E_{i\sim x} \left [\tilde \ell(j)\right ]=\E[\ell(j)],\quad \forall 1\le j\le K.
\end{equation*}
\end{lemma}
\begin{proof}
As the adversary is oblivious (or even stochastic),
\begin{equation*}
    \E_{i\sim x}\left [\tilde \ell(j)\right ]=\sum_{i=1}^K \frac{\E[\ell(j)]}{x_j}\mathbbm 1_{i=j}\cdot \Pr\{\mathbbm 1_i\}=\Pr\{i=j\}\frac{\E[\ell(j)]}{x_j}=\E[\ell(j)],
\end{equation*}

for any $1\le j\le K$.
\end{proof}

\begin{lemma}[FTRL Regret Decomposition]\label{lem:FTRL regret decomposition}
For any FTRL algorithm, i.e., the action $x_t$ for any $t\in [T]$ is decided by $\operatorname{argmin}_{x\in \triangle_{[K]}}(\eta_t\sum_{1\le s\le t}\langle \hat \ell_s,x\rangle+\Psi(x))$, where $\eta_t$ is the learning rate, $\hat \ell_s$ is some arbitrary vector and $\Psi(x)$ is a convex regularizer, we have
\begin{equation*}
    \sum_{t=1}^T \langle x_t-y,\hat \ell_t\rangle\le \sum_{t=1}^T (\eta_t^{-1}-\eta_{t-1}^{-1})(\Psi(y)-\Psi(x_t))+\sum_{t=1}^T \eta_t^{-1}D_\Psi(x_t,z_t)
\end{equation*}

for any $y\in \triangle_{[K]}$, where $z_t\triangleq \nabla \Psi^\ast(\nabla \Psi(x_t)-\eta_t\hat \ell_t)$.
\end{lemma}
\begin{proof}
Let $\hat L_t \triangleq \sum_{s=1}^t \hat \ell_s$, we then have
\begin{align*}
    \sum_{t=1}^T \langle x_t - y, \hat \ell_t\rangle & = \sum_{t=1}^T -\eta_t^{-1}\langle x_t,-\eta_t \hat \ell_t \rangle + \langle y, -\hat L_T\rangle \\
    & = \sum_{t=1}^T\eta_t^{-1}\left[\bar\Psi^*(-\eta_t \hat L_t) - \bar\Psi^*(-\eta_t \hat L_{t-1}) - \langle x_t,-\eta_t \hat \ell_t \rangle \right] \\
    &\quad + \sum_{t=1}^T \left[ \eta_t^{-1}\bar\Psi^*(-\eta_t \hat L_{t-1}) -\eta_t^{-1}\bar\Psi^*(-\eta_t \hat L_t) \right] + \langle y, -\hat L_T\rangle \\
    & = \sum_{t=1}^T \eta_t^{-1}D_{\bar\Psi^*}(-\eta_t\hat L_t,-\eta_t\hat L_{t-1}) + \sum_{t=1}^T \left[ \eta_t^{-1}\bar\Psi^*(-\eta_t \hat L_{t-1}) -\eta_t^{-1}\bar\Psi^*(-\eta_t \hat L_t) \right] + \langle y, -\hat L_T\rangle \\
    & = \sum_{t=1}^T \eta_t^{-1} D_{\Psi}(x_t, \nabla\bar\Psi^*(-\eta_t \hat L_t)) + \sum_{t=1}^T \left[ \eta_t^{-1}\bar\Psi^*(-\eta_t \hat L_{t-1}) -\eta_t^{-1}\bar\Psi^*(-\eta_t \hat L_t) \right] + \langle y, -\hat L_T\rangle \\
    & \stackrel{\text{(a)}}{\le} \sum_{t=1}^T \eta_t^{-1} D_{\Psi}(x_t, \nabla\Psi^*(-\eta_t \hat L_t)) + \sum_{t=1}^T \left[ \eta_t^{-1}\bar\Psi^*(-\eta_t \hat L_{t-1}) -\eta_t^{-1}\bar\Psi^*(-\eta_t \hat L_t) \right] + \langle y, -\hat L_T\rangle \\
    & = \sum_{t=1}^T \eta_t^{-1} D_{\Psi}(x_t, z_t) + \sum_{t=1}^T \left[ \eta_t^{-1}\bar\Psi^*(-\eta_t \hat L_{t-1}) -\eta_t^{-1}\bar\Psi^*(-\eta_t \hat L_t) \right] + \langle y, -\hat L_T\rangle \\
    & \stackrel{\text{(b)}}{=} \sum_{t=1}^T \eta_t^{-1} D_{\Psi}(x_t, z_t) + \sum_{t=1}^{T-1} \left[ \langle x_t, -\hat L_{t-1} \rangle -\eta_t^{-1}\Psi(x_t) -\sup_{x\in\Delta_{[K]}}\left\{ \langle x, -\hat L_t\rangle -\eta_t^{-1}\Psi(x)\right\} \right]\\
    &\quad + \langle x_T, -\hat L_{T-1} \rangle -\eta_T^{-1}\Psi(x_T) -\sup_{x\in\Delta_{[K]}}\left\{ \langle x, -\hat L_T\rangle -\eta_T^{-1}\Psi(x)\right\} + \langle y, -\hat L_T\rangle \\
    & \le \sum_{t=1}^T \eta_t^{-1} D_{\Psi}(x_t, z_t) + \sum_{t=1}^{T-1} \left[ \langle x_t, -\hat L_{t-1} \rangle -\eta_t^{-1}\Psi(x_t) - \langle x_{t+1}, -\hat L_t\rangle +\eta_t^{-1}\Psi(x_{t+1}) \right] \\
    &\quad + \langle x_T, -\hat L_{T-1} \rangle -\eta_T^{-1}\Psi(x_T) -\sup_{x\in\Delta_{[K]}}\left\{ \langle x, -\hat L_T\rangle -\eta_T^{-1}\Psi(x)\right\} + \langle y, -\hat L_T\rangle \\
    & = \sum_{t=1}^T \eta_t^{-1} D_{\Psi}(x_t, z_t) + \sum_{t=1}^T (\eta_{t-1}^{-1} - \eta_t^{-1})\Psi(x_t) - \sup_{x\in\Delta_{[K]}}\left\{ \langle x, -\hat L_T\rangle -\eta_T^{-1}\Psi(x)\right\} + \langle y, -\hat L_T \rangle \\
    & = \sum_{t=1}^T \eta_t^{-1} D_{\Psi}(x_t, z_t) + \sum_{t=1}^T (\eta_{t-1}^{-1} - \eta_t^{-1})\Psi(x_t) \\
    &\quad - \sup_{x\in\Delta_{[K]}}\left\{ \langle x, -\hat L_T\rangle -\eta_T^{-1}\Psi(x)\right\} + \langle y, -\hat L_T \rangle - \eta_T^{-1}\Psi(y) + \eta_T^{-1}\Psi(y) \\
    &\le \sum_{t=1}^T \eta_t^{-1} D_{\Psi}(x_t, z_t) + \sum_{t=1}^T (\eta_{t-1}^{-1} - \eta_t^{-1})\Psi(x_t) + \eta_T^{-1}\Psi(y) \\
    & = \sum_{t=1}^T \eta_t^{-1} D_{\Psi}(x_t, z_t) + \sum_{t=1}^T (\eta_t^{-1} - \eta_{t-1}^{-1})(\Psi(y) - \Psi(x_t))
\end{align*}
where step (a) is due to the Pythagoras property of Bregman divergences, and in step (b) we just plugged in the definition of $\bar\Psi^*$ in $\Psi$.
\end{proof}

\end{document}